\documentclass[]{article} 

\usepackage{amsmath,amssymb}
\usepackage{times}
\usepackage{graphicx}
\usepackage{color}
\usepackage{bbm}
\usepackage{latexsym}
\usepackage{dsfont}
\usepackage{verbatim}
\usepackage{booktabs,array}

\definecolor{bl}{RGB}{30,30,150}
\usepackage{hyperref}
\hypersetup{colorlinks=true, citecolor=bl, linkcolor=bl, urlcolor=bl}

\usepackage{authblk}
\usepackage{amsthm}

\theoremstyle{plain}
\newtheorem{theorem}{Theorem}
\newtheorem{lemma}[theorem]{Lemma}
\newtheorem{corollary}[theorem]{Corollary}
\newtheorem{proposition}[theorem]{Proposition}
\theoremstyle{remark}

\theoremstyle{definition}
\newtheorem{example}[theorem]{Example}
\newtheorem{definition}[theorem]{Definition}

\newcommand{\sm}[1]{\begin{smallmatrix} #1 \end{smallmatrix}}

\newcommand{\Acal}{\mathcal{A}}

\newcommand{\Ecal}{\mathcal{E}}

\newcommand{\Mcal}{\mathcal{M}}

\newcommand{\Xcal}{\mathcal{X}}
\newcommand{\Ycal}{\mathcal{Y}}

\newcommand{\PP}{\mathbb{P}}
\newcommand{\R}{\mathbb{R}}

\newcommand{\conv}{\ensuremath{\operatorname{conv}}}
\newcommand{\rank}{\ensuremath{\operatorname{rank}}}
\newcommand{\supp}{\operatorname{supp}}
\newcommand{\argmax}{\operatorname{argmax}}

\newcommand{\TMcal}{\Mcal^{\text{\normalfont tropical}}}

\newcommand{\Amd}{\mathfrak{A}}
\newcommand{\Kmd}{\mathfrak{K}}

\newcommand{\keywords}[1]{\smallskip\par\noindent{\small{\em Keywords\/}: #1}}
\newcommand{\subclass}[1]{\smallskip\par\noindent{\small{\em MSC2010\/}: #1}}

\graphicspath{ {./Figures/}}

\begin{document}

\title{\Large\bf Dimension of Marginals of Kronecker Product Models\\\medskip
	\large Geometry of hidden-visible products of exponential families\bigskip}

\author[1]{{\normalsize\bf Guido Mont\'ufar}\thanks{montufar@mis.mpg.de}}
\author[2]{{\normalsize\bf Jason Morton}\thanks{morton@math.psu.edu}}

\affil[1]{\small Max Planck Institute for Mathematics in the Sciences, Inselstra\ss e~22, 04103 Leipzig, Germany}
\affil[2]{\small Department of Mathematics, Pennsylvania State University, PA 16802, University Park, USA}

\date{\small\today}

\maketitle

\begin{abstract}
\noindent	A Kronecker product model is the set of visible marginal probability distributions of an exponential family whose sufficient statistics matrix factorizes as a Kronecker product of two matrices, one for the visible variables and one for the hidden variables. 
	We estimate the dimension of these models by the maximum rank of the Jacobian in the limit of large parameters. 
	The limit is described by the tropical morphism; a piecewise linear map with pieces corresponding to slicings of the visible matrix by the normal fan of the hidden matrix. 
	We obtain combinatorial conditions under which the model has the expected dimension, 
	equal to the minimum of the number of natural parameters and the dimension of the ambient probability simplex. 
	Additionally, we prove that the binary restricted Boltzmann machine always has the expected dimension. 
	\keywords{expected dimension, tropical geometry, secant variety, restricted Boltzmann machine, inference function, Kronecker product, Hadamard product, Khatri-Rao product, exponential family harmonium}
	\subclass{%
		14T05, 
		52B05  
	}
\end{abstract}

\section{Introduction}

Simple probability distributions are often composed in order to obtain more interesting or complex probability distributions. 
Natural compositions include tensor products, convex combinations, and renormalized entrywise products.  Stochastic neural networks, for example, define compositions of elementary probability distributions on the states of individual neurons, which result in interesting joint probability distributions on the states of subsets of neurons of the entire network. 
We study probability models defined by building the Kronecker product of the sufficient statistics matrices of two exponential families and marginalizing over one of the two. 
We call these probability models hidden-visible Kronecker products of exponential families, or simply Kronecker product models. 
Examples include mixtures of exponential families, 
restricted Boltzmann machines~\cite{Freund1992}, 
and, more generally, hierarchical models on Cartesian products of simplicial complexes of hidden and visible variables. 
A related class of probability models is known in machine learning under the name exponential family harmonium~\cite{welling:exponential}.

We are interested in the dimension of these compositions, depending on the properties of the visible and hidden factors. 
Marginalization is in general a non-injective map which may collapse the dimension of the probability model that is being marginalized. 
Well known examples of this behavior are mixture models, which may have a dimension that is strictly smaller than the dimension of the corresponding non-marginalized models and the ambient probability simplex. 
Many interesting mixture models correspond to secant varieties and as such their dimension has been subject of numerous works in algebraic geometry. 
Examples include the secants of Segre products~\cite{AboOttavianiPeterson,Catalisano2011,raicu2012secant} and Veronese varieties~\cite{Abo,raicu2012secant}. Reference texts are~\cite{zaktangents,landsbergtensors}. 
Computing the dimension of secant varieties is a notoriously difficult problem, even in cases where the basis variety is of striking simplicity, such as the Segre products of one-dimensional projective spaces (corresponding to binary independence models in statistics or to the decomposable $2\times \cdots\times 2$ tensors in signal processing). 
In turn, computing the dimension of marginals of general exponential families seems near to hopeless. 
In certain cases, however, the dimension of secant varieties can be estimated by solving linear programs. 
The main idea has been illuminated in the tropical approach to secant dimensions by Draisma~\cite{Draisma}. 
Depending on the basis variety, the dimension estimates resulting from this approach give in fact the exact dimension. 
For Segre varieties, 
a sufficient condition is the existence of error correcting codes of a certain cardinality. 
Following these ideas, the dimension of binary and discrete restricted Boltzmann machines (Hadamard products of secant varieties of Segre products) have been studied in~\cite{Cueto2010} and~\cite{montufar2013discrete}, respectively. 
The goal of the present paper is to elucidate the approach from the perspective of exponential families, especially hierarchical models, and to provide explicit results for Kronecker product models (Definition \ref{def:kpm}) generalizing some of the above mentioned work on secants and restricted Boltzmann machine dimensions.

Consider two finite integers $N,M\in\mathbb{N}$ and a probability model $\{ p_\theta \colon \theta\in\R^M \}\subseteq\Delta_{N-1}$, parametrized by a 
function $\phi\colon \R^M\to\Delta_{N-1};\; \theta\mapsto p_\theta$ from $M$-dimensional Euclidean space $\R^M$ to the $(N-1)$-dimensional probability simplex 
$\Delta_{N-1}:=\{p\in\R^N \colon p(x)\geq 0\text{ for all } x \text{ and}$ $\sum_{x=1}^N p(x)=1 \}$. 
The dimension of $\phi(\R^M)$ is given by the maximum rank of the Jacobian $J_\phi$ of the map $\phi$ over the points of $\R^M$ where $\phi$ is smooth. 
Hence the problem reduces to  computing the maximum of the rank function over a parametric set of matrices. 
In general, computing the rank of these matrices is difficult. 
A possible approach  is to maximize the rank only over a subset of parameters where the Jacobian is simpler. 
When the result matches the maximum possible value, equal to the dimension $M$ of the parameter domain $\R^M$ or to the dimension $N-1$ of the codomain $\Delta_{N-1}$, 
we can be sure that we have attained the global maximum.

The rank becomes tractable when there is an obvious way to transform the Jacobian into a suitable block matrix by elementary matrix operations. 
Consider the case where $\phi$ parametrizes an exponential family, i.e., $\phi(\theta) = \exp (\langle \theta, F\rangle - \psi(\theta))$ for some matrix $F\in\R^{M\times N}$ of sufficient statistics, where $\psi$ is the normalizing (log-partition) function. 
Note that $\phi(\alpha\theta)\propto \phi^\alpha(\theta)$ for each $\alpha\in\R$. 
Hence, in the limit of large parameters ($\alpha\to\infty$), the resulting probability vector is proportional to the indicator of $\argmax\langle\theta,F\rangle$. 
The optimizing set is piecewise constant on~$\theta$ and usually has a small cardinality. 
Accordingly, if $\phi$ parametrizes marginals of an exponential family, then, in the limit of large $\alpha$, the matrix $J_\phi(\alpha\theta)$ has a block structure and its rank may be easy to determine. 
The mathematical formalism describing this limit of large parameters is known as tropical geometry. 
The tropicalization of $\phi$ produces a piecewise linear version of $\phi$ that captures its combinatoric properties. 
In this approach, the dimension of the models under consideration can be related to polyhedral optimization problems. 
This leads to combinatorial problems about optimal slicings of point configurations by polyhedral fans of point configurations. 
Characterizing the sets of maximizers of vectors in the row span of $F$ and the combinatorics of the convex support of the exponential family (the convex hull of the columns of $F$) is not an easy task in general. 
However, we do not need to fully solve that problem in order to estimate the maximum rank. 

As mentioned above, we will place emphasis on Kronecker product models with factors given by hierarchical models. Hierarchical models are ubiquitous in statistics applications~\cite{Wainwright:2008}. 
An algebraic perspective on hierarchical graphical models was given in~\cite{geiger2006}. 
Some works have discussed the convex support polytopes of hierarchical models, showing relations to linear codes and oriented matroids~\cite{Kahle2009,RKA10:Support_Sets_and_Or_Mat} and describing properties such as neighborliness and simpliciality~\cite{Kahle2010,montufar2013mixture}. 

\medskip

This paper is organized as follows. 
Section~\ref{section:jacobian} provides basic definitions of marginals of discrete exponential families and discusses the Jacobian of their natural parametrization. 
It also introduces the associated inference functions and tropical models. 
Section~\ref{section:model} defines the Kronecker product model and discusses its general properties. 
The sections that follow are devoted to specific types of Kronecker product models. 
Section~\ref{section:interactionmodels} gives a brief introduction to hierarchical models and defines related types of error correcting codes. 
Section~\ref{section:secants} studies the tropical dimension of mixtures of hierarchical models. 
Section~\ref{section:Hadamard} studies the tropical dimension of Hadamard products of mixtures of hierarchical models. 
Section~\ref{section:harmonium} studies the general case of Kronecker product hierarchical models. 
Theorem~\ref{theorem:generalKronecker} from Section~\ref{section:harmonium} includes the Theorems~\ref{theorem:mixtureskinteraction} and Theorem~\ref{theorem:indhidintvis} from Sections~\ref{section:secants} and~\ref{section:Hadamard} as special cases. 
Section~\ref{section:binRBM} proves that binary restricted Boltzmann machines always have the expected dimension, thereby solving the dimension question for cases that were left open in~\cite{Cueto2010}. 
Section~\ref{section:conclusion} discusses the results.

\section{Marginals of Exponential Families}
\label{section:jacobian}
In this section we present basic definitions of exponential families and their marginals. We discuss the Jacobian of the natural parametrization and relate its behavior in the limit of large parameters with the inference functions of the model. 
This leads us to the definition of the tropical morphism and a simplified rank estimation problem. 

Consider two finite sets $\Xcal$ and $\Ycal$. A probability distribution on $\Xcal\times\Ycal$ is a real-valued vector $p\in\R^{\Xcal\times\Ycal}$ with entries $p(x,y)\geq0$, $(x,y)\in\Xcal\times\Ycal$, satisfying $\sum_{x,y}p(x,y)=1$. 
Consider a function $F\colon \Xcal\times\Ycal\to \R^d$. 

\begin{definition}
	\label{def:exponentialfamily}
	The exponential family $\Ecal_F$ with sufficient statistics $F$ consists of all probability distributions on $\Xcal\times\Ycal$ of the form 
	\begin{equation*}
	p_\theta(x,y)=\frac{1}{Z(\theta)}\exp(\langle\theta, F(x,y)\rangle),\quad (x,y)\in\Xcal\times\Ycal, \quad \theta\in\R^d, 
	\end{equation*}
	where $\langle\cdot,\cdot\rangle$ denotes the standard inner product and $Z\colon \theta \mapsto \sum_{x',y'} \exp(\langle\theta, F(x',y')\rangle )$ is a normalization function. 
	Note that each distribution in the exponential family has strictly positive entries. 
	We will regard $F$ as a matrix with columns $F(x,y)\in \R^d$, $(x,y)\in \Xcal\times\Ycal$. 
	We note that the exponential family $\Ecal_F$ is fully characterized by the row space $\langle \R^d, 
	F\rangle \subseteq \R^{\Xcal\times\Ycal}$ of the sufficient statistics matrix $F$. 
	More precisely, two matrices $F$ and $G$ produce the same exponential family if and only if 
	$(F; \mathds{1})$ and $(G;\mathds{1})$ have the same row span, where $\mathds{1}$ is a row of ones. 
	From now on we will always assume, without loss of generality, that $F$ includes $\mathds{1}$ in its row span. 
	The dimension of the exponential family is then $\dim(\Ecal_F)=\rank(F) -1$. 
	This is the same as the dimension of the convex support polytope $\conv\{ F(x,y)\colon (x,y)\in\Xcal\times\Ycal \}$.  
\end{definition}

\begin{definition}
	\label{def:marginal}
	The marginal model $\Mcal_F$ on $\Xcal$ of the exponential family $\Ecal_F$ is the set of all probability distributions of the form  
	\begin{equation*}
	p_\theta(x) = \sum_{y\in\Ycal} \frac{1}{Z(\theta)}\exp(\langle \theta, F(x,y) \rangle ),\quad x\in\Xcal, \quad \theta\in\R^d. 
	\end{equation*} 
	The marginal model $\Mcal_F$ is the image of $\Ecal_F$ by the marginalization map, which is the linear map represented by the matrix with rows equal to the indicators  $\mathds{1}_{x}\in\R^{\Xcal\times\Ycal}$ of $\{x\}\times\Ycal$, for each $x\in\Xcal$. 
	We are interested in the dimension of $\Mcal_F$. 
	When $\dim(\Mcal_F)=\min\{\dim(\Ecal_F) , |\Xcal|-1 \}$, we say that $\Mcal_F$ has the expected dimension, meaning that marginalization does not collapse the dimension of $\Ecal_F$, or that the marginal $\Mcal_F$ is full dimensional, having the same dimension as the simplex $\Delta_{|\Xcal|-1}$ of all probability distributions on~$\Xcal$.  
\end{definition}

The dimension of $\Mcal_F$ is equal to the maximum rank of the Jacobian matrix of the parametrization $\theta\mapsto (p_\theta(x))_x$. 
The Jacobian is given by  
\begin{equation}
J_{\Mcal_F}(\theta) = \left( \sum_y p_\theta(x,y) F(x,y) - \sum_y p_\theta(x,y) \sum_{x',y'} p_\theta(x',y') F(x',y') \right)_x. 
\label{eq:derivative}
\end{equation}
The second term corresponds to the normalization function $Z$. 
For the rank we have 
\begin{align}
\rank \left(J_{\Mcal_F}(\theta)\right) 
= & \rank \left( \sum_y p_\theta(x,y) F(x,y) \right)_x  -1 \nonumber \\ 
= & \rank \left( \sum_y p_\theta(y|x) F(x,y) \right)_x -1.  \label{eq:rankJ} 
\end{align}
The first equality follows from the assumption that $F$ has a constant row. 
The second one is because $p_\theta(x)=\sum_{y'}p_\theta(x,y')> 0$ for all $x$. 
Here $p_\theta(y|x) : = p_\theta(x,y)/\sum_{y'}p_\theta(x,y')$ denotes the conditional probability of $y$ given $x$. 
A geometric interpretation is that $\rank(J_{\Mcal_F}(\theta))$ is the dimension of the polytope defined as the convex hull of 
\begin{equation*}
\sum_{y} p_\theta(y|x) F(x,y), \quad  x\in \Xcal. 
\end{equation*}

Evaluating Equation~\ref{eq:rankJ} is difficult, in general. 
The problem is easier in the limit of large parameters, 
where the sum over $y$ almost always reduces to a single term. 
To see that this is the case, note that multiplicative factors of the parameter $\theta$ correspond to exponential factors of the probability distribution, 
such that $p_{\alpha\theta}(\cdot|x)\propto p_\theta(\cdot|x)^\alpha$. 
Therefore, for any $\theta\in\R^d$, the limit $\lim_{\alpha\to\infty} p_{\alpha\theta}(y|x)$ is non-zero only for $y \in \argmax_y p_\theta(y|x) = \argmax_y \langle \theta, F(x,y)\rangle$. 
Following this line of thought, it is convenient to define the function that outputs the most likely value of $y$ to any given~$x$:  

\begin{definition}
	The inference function of $\Mcal_F$ with parameter $\theta\in\R^d$ is given by 
	\begin{equation*}
	h_\theta\colon \Xcal\to 2^\Ycal;\; x\mapsto  h_\theta(x) = \argmax_y \langle \theta, F(x,y)\rangle. 
	\end{equation*}
	Here $2^\Ycal$ denotes the power set of $\Ycal$. 
	Geometrically, $h_\theta(x)$ 
	is the set of $y\in\Ycal$ for which $F(x,y)$ lies in the supporting hyperplane of $F(x,y)$, $y\in\Ycal$, with normal~$\theta$. 
	The situation is illustrated in Figure~\ref{fig:inf}. 
\end{definition}

\begin{figure}[t]
	\centering
	\includegraphics{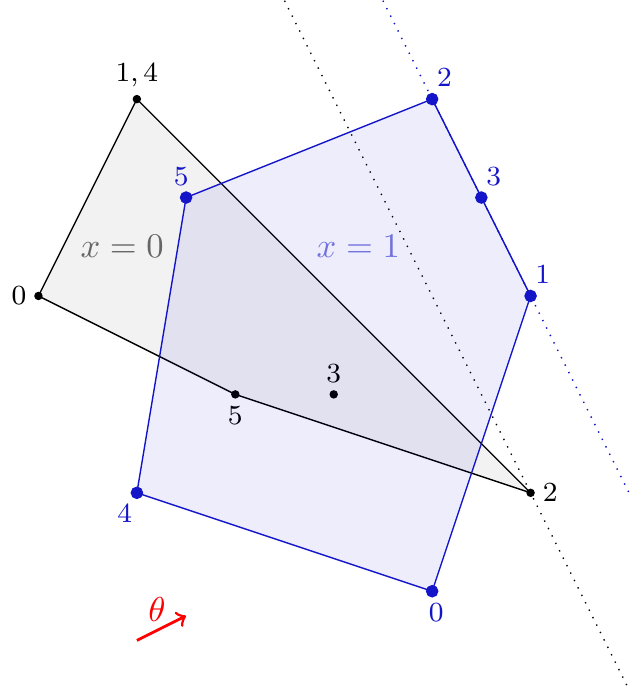}
	\caption{Illustration of an inference function for a model with $\Xcal=\{0,1\}$, $\Ycal=\{0,\ldots, 5\}$. 
		Each dot is a vector $F(x,y)$. Small dots correspond to $x=0$ and large ones to $x=1$. The value of $y$ is indicated next to each dot. 
		For the illustrated choice of $\theta$ the inference function is $h_\theta(x=0) = \{2\}$, $h_\theta(x=1) = \{1,2,3\}$ and the tropical morphism is given by $\Acal_\theta = (F(0,2), \frac13 (F(1,1) +F(1,2) +F(1,3) ))$. 
		For generic choices of $\theta$ the vector $F(x,h_\theta(x))$ is a vertex of the polytope $\conv \{F(x,y)\colon y\in \Ycal\}$, for all $x$. 
	}
	\label{fig:inf}
\end{figure}

We have the following dimension bounds: 
\begin{proposition}
	\label{proposition:rankJacobian}
	The dimension of the marginal model $\Mcal_F$ satisfies 
	\begin{equation*}
	\rank(F)-1 = \dim(\Ecal_F)\geq 
	\dim(\Mcal_F) 
\geq  \max_{\theta}\rank \left( \bar F(x,h_\theta(x)) \right)_x -1,
	\end{equation*} 
	where $\bar F(x,h_\theta(x)) := \frac{1}{|h_\theta(x)|}\sum_{y\in h_\theta(x)} F(x,y)$. 
\end{proposition}

\begin{proof}
	We have 
	\begin{align*}
	\max_{\theta\in\R^d} \rank \left(J_{\Mcal_F}(\theta)\right) 
	= & \max_{\theta\in\R^d} \rank \left(\sum_y p_\theta(y|x) F(x,y) \right)_x  -1\\
	\geq & \max_{\theta\in\R^d} \rank \left( \lim_{\alpha\to\infty} \sum_y p_{\alpha\theta}(y|x) F(x,y) \right)_x  -1\\
	=& \max_{\theta\in\R^d} \rank \left( \sum_{y\in h_\theta(x)} \frac{1}{|h_\theta(x)|} F(x,y) \right)_x  -1.  \end{align*}
	The first line is Equation~\ref{eq:rankJ}. 
	The second line follows from the continuity of the parametrization of the exponential family (see Definition~\ref{def:marginal}) and the lower semicontinuity of the rank function. 
	The third line is because $\lim_{\alpha\to\infty} p_{\alpha\theta} (y|x)$ is positive and constant on $h_\theta(x)=\argmax_y p_{\theta}(y|x)$ and zero on $\Ycal\setminus h_\theta(x)$. 
\end{proof}

Proposition~\ref{proposition:rankJacobian} shows that we can estimate the dimension of the marginal model by studying the maximum rank over $\theta$ of the piecewise constant matrix-valued function 
\begin{equation*}
\Acal_\theta: = (\bar F(x,h_{\theta}(x)))_x.
\end{equation*} 
For each $x\in\Xcal$, the column $\bar F(x, h_\theta(x))$ is the average of the maximizers of the linear form $\langle \theta, \cdot \rangle$ over $F(x,y)$, $y\in \Ycal$. 
For generic choices of $\theta\in\R^d$, the set of maximizers $F(x,y)$, $y\in h_\theta(x)$, consists of one single element, for each $x$. 
In particular, when all vectors $F(x,y)$ are different, for generic choices of $\theta$, the inference function $h_\theta$ maps each $x$ to a single~$y$.

The vectors $\bar F(x,h_\theta(x)), x\in \Xcal$, do not necessarily lie on the same supporting hyperplane of $F(x,y)$, $(x,y)\in\Xcal\times\Ycal$, although the converse is true in the following sense. 
If there is supporting hyperplane that intersects $F(x,y)$, $(x,y)\in\Xcal\times\Ycal$, exactly at $F(x,f(x))$, $x\in \Xcal$, for some $f\colon \Xcal\to\Ycal$, then $f$ is an inference function. 
However, since these vectors lie on a supporting hyperplane (which usually defines a proper face of the convex support polytope), they are not suited for estimating the maximum rank.

The matrix $\Acal_\theta$ defines a geometric object called the tropical version of the original marginal model: 
\begin{definition}
	\label{definition:tropicalmorphism}
	The tropical version of $\Mcal_F$, denoted $\TMcal_F$, is the set of all vectors in $\R^\Xcal/\R \mathds{1}$ (modulo addition of constant vectors) of the form 
	\begin{equation*}
	\Phi_\theta(x) = \langle\theta, \bar F(x, h_\theta(x)\rangle,\quad x\in\Xcal,  \quad \text{parametrized by }\theta\in\R^d. 
	\end{equation*} 
\end{definition}

Proposition~\ref{proposition:rankJacobian} can be regarded as a version of the Bieri-Groves theorem~\cite{Bieri1984,Draisma}, 
stating that the dimension of the marginal model is bounded below by the dimension of its tropical version: 
\begin{equation*}
\dim(\Mcal_F)\geq \dim(\TMcal_F). 
\end{equation*}
In particular, when $\TMcal_{F}$ has the expected dimension, then $\Mcal_F$ also has the expected dimension. 
This is the central idea of~\cite{Draisma} and subsequently~\cite{Cueto2010,montufar2013discrete} for estimating the dimension of secant varieties and restricted Boltzmann machines. 

We note that the marginal model and its tropical version are independent of the sufficient statistics matrix used to parametrize the underlying exponential family: 
\begin{proposition}
	\label{proposition:observations1}
	If $E = Q^\top F$ is a non-singular linear transformation of $F$, 
	then $\Ecal_E = \Ecal_F$, $\Mcal_E = \Mcal_F$, and 
	$\TMcal_{E} = \TMcal_F$. 
	More generally, if $G= R^\top F$ is a linear transformation of $F$, 
	then $\Ecal_G\subseteq\Ecal_F$, $\Mcal_G\subseteq\Mcal_F$, and $\TMcal_G\subseteq\TMcal_F$. 
\end{proposition}

\begin{proof}
	The equality of the exponential families follows from the equality of the row spaces of $F$ and $E$. 
	For the tropical models note that for each $\theta$ there is a $\vartheta = Q^{-1} \theta$ with $\langle \theta, F(x,h_\theta(x))\rangle = \langle \vartheta, E(x, h_{\vartheta}(x))\rangle$. 
	For the inclusions note that the row space of $G= R^\top F$ is a linear subspace of the row space of $F$. 
\end{proof}

\section{The Kronecker Product Model}
\label{section:model}

Kronecker product models are marginals of exponential families whose sufficient statistics matrix factorizes as the Kronecker product of a sufficient statistics matrix over the visible states and one over the hidden states. 
Recall the definition of the Kronecker product $(A_{i,j})_{i,j}\otimes (B_{k,l})_{k,l}:=(A_{i,j} B_{k,l})_{(i,k),(j,l)}$. 

\begin{definition}\label{def:kpm}
	A {\em Kronecker product model} is a marginal model $\Mcal_F$, where $F$ factorizes as $F(x,y)=A(x)\otimes B(y)$, $x\in\Xcal$, $y\in\Ycal$, for some $A\in\R^{a\times \Xcal}$ and $B\in\R^{b\times \Ycal}$. 
\end{definition}

We will use the notations $A=(A_x)_x\in\R^{a\times\Xcal}$, $B=(B_y)_y\in\R^{b\times\Ycal}$, 
and assume that the row space of each matrix contains a constant non-zero vector. 
For simplicity in the following we assume that all columns of $B$ are different. 
The general case with repeated columns is very similar, but needs more complicated notations.

Consider a generic choice of the parameter $\theta$, 
such that the inference function $h_\theta$ maps each $x\in\Xcal$ to a single $y\in\Ycal$. 
The tropical morphism 
$\Phi_\theta(x) = \langle \theta,\Acal_\theta \rangle$ of a Kronecker product model has the form 
\begin{equation}
\Acal_\theta = ( F(x,h_\theta(x)))_x = (A_x\otimes B_{h_\theta(x)})_x = A\odot B_{h_\theta}, 
\label{eq:tropicalKronecker}
\end{equation}
where $(A_{i,j})_{i,j}\odot (B_{k,l})_{k,l}:=(A_{i,j} B_{k,j})_{(i,k),j}$ denotes the column-wise Kronecker product or Khatri-Rao product~\cite{KhatriRao1968}. 
Alternatively, after rearranging columns, we can write this as
\begin{equation*}
\Acal_\theta = (B_y \otimes A_{h^{-1}_\theta(y)})_y = B\odot A_{h^{-1}_\theta}, 
\end{equation*} 
where now $\odot$ denotes the column-block-wise Kronecker product. 
Here the column-blocks, indexed by $y$, are $B_y$ and $A_{h^{-1}_\theta(y)}$. 
When $h^{-1}_\theta(y)=\emptyset$, we simply omit the block $B_y\otimes A_{h^{-1}_\theta(y)}$.

Unfortunately there is no formula for expressing the rank of a Khatri-Rao product in terms of the ranks of its factors. 
A simple lower bound for matrices consisting of non-zero columns is $\rank(A\odot B_h) \geq \max \{ \rank(A), \rank(B_h)\}$. 
More elaborate lower bounds can be given in terms of Kruskal ranks~\cite{890366}, also for column-block partitioned matrices~\cite{doi:10.1137/060661685}. 
Our analysis seeks to characterize pairs of matrices $A$ and $B$ for which the upper bound 
$\rank(A\odot B_h)\leq \rank(A\otimes B_h)=\rank (A)\cdot\rank (B_h)$ is attained for some inference function $h$. 
For this it is critical to study the possible inference functions.

The factorization property $F=A\otimes B$ leads to highly structured inference functions. 
We explain this in the following. 
For a given parameter vector $\theta=(\theta_{(i,j)})_{(i,j)}\in\R^{a b }$ let $\Theta=(\theta_{(i,j)})_{j,i}\in\R^{b \times a}$ denote the matrix with column-by-column vectorization equal to $\theta$. 
By Roth's lemma~\cite{Roth1934} we have the following equalities: 
\begin{gather*}
\langle\theta, (A\otimes B)_{(x,y)}\rangle =\langle \Theta A_x, B_y\rangle= \langle \Theta^\top  B_y, A_x \rangle\quad\text{for all } x\in\Xcal, y\in\Ycal . 
\label{eq:rotheqs}
\end{gather*} 
In turn, the same inner product describes the following distributions over $(x,y)$, $y$, and $x$: 
\begin{align}
p_\theta(\cdot,\cdot) =& \frac{1}{Z(\theta)} \exp(\langle\theta, A\otimes B \rangle) \;\in\;\Ecal_{A\otimes B},\nonumber\\
p_\theta(\cdot|x)  =& \frac{1}{Z(\Theta A_x)} \exp(\langle \Theta A_x, B \rangle) \;\in\;\Ecal_{B}, \label{equation:conditionaldistributions}\\ 
p_\theta(\cdot|y)  =& \frac{1}{Z(\Theta^\top B_y)} \exp(\langle \Theta^\top  B_y, A \rangle) \;\in\;\Ecal_A.\nonumber 
\end{align}
In particular, $p_\theta(\cdot|x)$ appears in the Equation~\ref{eq:rankJ} of the Jacobian rank. 
Geometrically, $\Theta A$ is the linear projection of the columns of $A$ by the matrix $\Theta$ to the parameter space of the hidden exponential family $\Ecal_B$. 
Similarly, $\Theta^\top  B$ is the projection of the columns of $B$ by the matrix $\Theta^\top$ to the parameter space of the visible exponential family  $\Ecal_A$. 

The inference function of a Kronecker model satisfies 
\begin{align*}
h_\theta(x)
=&  \operatorname{argmax}_y  p_\theta(y|x) 
= \operatorname{argmax}_y  \langle \theta, A_x \otimes B_y\rangle   \\
=&  \operatorname{argmax}_y  \langle \Theta A_x, B_y \rangle 
=  \{y\in\Ycal \colon \Theta A_x \in N_{B}(y)\} \\
=&  \operatorname{argmax}_y  \langle A_x, \Theta^\top B_y \rangle 
=  \{y\in\Ycal \colon A_x \in N_{\Theta^\top B}(y)\}. 
\end{align*}
Here the normal cones of $B$ at $B_y$ and of $\Theta^\top B$ at $\Theta^\top B_y$ are defined, respectively, as
\begin{align*}
N_B(y) := & \{r\in\R^b \colon \langle r, B_y - B_{y'}\rangle \geq 0 \text{ for all } y'\in\Ycal\setminus \{y\}\}, \\
N_{\Theta^\top B}(y) :=&  \{r\in\R^a \colon \langle r, \Theta^\top B_y - \Theta^\top B_{y'}\rangle \geq 0 \text{ for all } y'\in\Ycal\setminus \{y\}\}.  
\end{align*} 

In turn, the inference function $h_\theta$ can be interpreted as a slicing of $\Theta A$ by the normal fan of $B$, 
or, equivalently, a slicing of $A$ by the normal fan of $\Theta^\top B$: 
\begin{definition}
	\label{definition:slicing}
	A $B$-slicing of $A$ is a partition of the column set of $A$ into the blocks 
	$A_{C_y}:= (A_x)_{x\in C_y}$ with $C_y:=h_\theta^{-1}(y) = \{x\in\Xcal \colon h_\theta(x)=y\}$, for all $y\in\Ycal$. 
	Here we assume that $h_\theta$ maps each $x$ to a single $y$, which is the generic case when all columns of $B$ are different. 
	Note that some of the sets $C_y = h^{-1}_\theta(y)$ may be empty. 
\end{definition}

Related to the statements of Proposition~\ref{proposition:observations1} we have the following observations:  
If $B' = D B$ for some matrix $D$, then any $B'$-slicing is also a $B$-slicing. 
Furthermore, if $B'$ is a matrix consisting of all columns of $B$ that lie in a common supporting hyperplane of $B$, then any $B'$-slicing is a $B$-slicing. 
Note also that, by the mixed-product property of the Kronecker product, 
$(C\otimes D)(A\otimes B)$ $=(CA)\otimes (DB)$, 
the Kronecker product of two linearly transformed matrices $CA$ and $DB$ is given by a linear transformation of the Kronecker product of the two original matrices.

\section{Hierarchical Models} 
\label{section:interactionmodels}

We are interested in Kronecker product models for which the factor exponential families $\Ecal_A$ and $\Ecal_B$ are hierarchical models. 
In this section we provide the necessary definitions. 

Consider $n$ random variables with finite state sets $\Xcal_i=\{0,1\ldots,|\Xcal_i|-1\}$, $i\in[n]:=\{1,\ldots, n\}$. 
We write $x=(x_1,\ldots, x_n)$ for an element of $\Xcal=\Xcal_1\times\cdots\times\Xcal_n$. 
Given some $x\in\Xcal$ and $\lambda\subseteq[n]$, we write $x_\lambda =(x_i)_{i\in \lambda}$ for the natural projection of $x$ to the $\lambda$-coordinates. 	
Unless otherwise stated, in all that follows $\Lambda$ will denote an inclusion closed set of subsets of $[n]$, 
such that $\lambda\in \Lambda$ and $\lambda'\subseteq\lambda$ imply $\lambda'\in \Lambda$. 
Consider the linear subspace of $\R^\Xcal$ consisting of all linear combinations of real-valued functions that depend only on $x_\lambda$, $\lambda\in\Lambda$, 
\begin{equation*}
V_\Lambda(\Xcal):=\Big\{ \sum_{\lambda\in\Lambda}f_\lambda\colon f_\lambda(x)= f_\lambda(x_\lambda) \Big\}\subseteq\R^\Xcal.
\end{equation*}

\begin{definition}
	The hierarchical model on $\Xcal$ with interactions $\Lambda$ is the exponential family 
	$\Ecal_{A}$, where $A\in\R^{a\times\Xcal}$ is a matrix with row span $V_\Lambda(\Xcal)$. 
	We will denote this model by $\Ecal_\Lambda$. 
	An important special case is the $k$-interaction model $\Ecal_{\Lambda_k}$, 
	defined by some $1\leq k\leq n$ and $\Lambda_k=\{\lambda\subseteq [n]\colon |\lambda|\leq k\}$. 
\end{definition}

An important example is the independence model $\Ecal_{\Lambda_1}$, which is the $k$-interaction model with $k=1$.  
The independence model consists of probability distributions that factorize as $p(x_1,\ldots, x_n) = p_1(x_1)\otimes \cdots\otimes p_n(x_n)$, 
where, for each $i\in[n]$, $p_i$ is a probability distribution over $x_i$. 
This model corresponds to the Segre embedding of $\PP^{|\Xcal_1|-1}\times\cdots\times \PP^{|\Xcal_n|-1}$ into $\PP^{\prod_{i\in[n]} |\Xcal_i| -1 }$.

The sufficient statistics matrix of a hierarchical model can be constructed in the following simple way. 
The first row, with index $\emptyset$, is the constant vector of ones $A_{\emptyset,x}:=1$, $x\in\Xcal$. 
The other rows are indexed by pairs $(\lambda, \tilde x_\lambda)$, with $\lambda\in\Lambda$ and $\tilde x_\lambda\in \tilde\Xcal_\lambda:=\times_{i\in\lambda}(\Xcal_i\setminus\{0\})$. 
The $((\lambda, \tilde x_\lambda),x)$-th entry of the matrix is 
\begin{equation}
A_{(\lambda,\tilde x_\lambda),x} := 
\begin{cases}
1,&\text{if }x_\lambda=\tilde x_\lambda\\
0,&\text{otherwise}  
\end{cases} .
\label{eq:suffstatinteraction}
\end{equation}

\begin{proposition}
	\label{lemma:row-space}
	The matrix $A$ from Equation~\ref{eq:suffstatinteraction} has row space $V_\Lambda(\Xcal)$. 
	Furthermore, $\dim(\Ecal_\Lambda)= \rank(A)-1 = \dim(V_\Lambda(\Xcal))-1 =\sum_{\lambda\in\Lambda\setminus\emptyset} \prod_{i\in\lambda}(|\Xcal_i|-1)$. 
\end{proposition}

The statement of Proposition~\ref{lemma:row-space} is well known in the context of hierarchical models. 
One way of proving it is as follows: 

\begin{proof}
	Let $\lambda\in\Lambda$ and let $f_\lambda\in\R^\Xcal$ be a function with $f_\lambda(x)=f_\lambda(x_\lambda)$ for all $x\in\Xcal$. 
	Note that $f_\lambda$ can be written as a linear combination of indicator functions of cylinder sets as 
	$f_\lambda(x) = \sum_{x^\ast_\lambda\in\Xcal_\lambda} f_\lambda(x^\ast_\lambda) \mathds{1}_{\{ x'\in\Xcal\colon x'_\lambda = x^\ast_\lambda \}}(x)$ for all $x\in\Xcal$.  
	Hence we only need to show that the row span of $A$ contains the indicator functions  $\mathds{1}_{\{ x' \in\Xcal \colon x'_\lambda = x_\lambda^\ast\}}$ for all $x_\lambda^\ast\in\Xcal_\lambda$. 
	For any $x^\ast_\lambda$ and $\lambda^\ast = \supp(x^\ast)\cap\lambda$ we have that 
	\begin{equation*} 
	\mathds{1}_{\{ x'\in\Xcal \colon x'_\lambda = x_\lambda^\ast \}}(x) = 
	\sum_{\lambda^\ast \subseteq \lambda' \subseteq\lambda }  
	(-1)^{|\lambda' \setminus \lambda^\ast|} 
	\sum_{\substack{\tilde x_{\lambda'}\in\tilde\Xcal_{\lambda'} \colon \\ \tilde x_{\lambda^\ast} =x_{\lambda^\ast}^\ast} } A_{(\lambda',\tilde x_{\lambda'}),x},\quad\text{for all $x\in\Xcal$}.
	\end{equation*} 
	Hence the row span of $A$ contains the indicator function $\mathds{1}_{\{ x'\in\Xcal \colon x'_\lambda = x_\lambda^\ast \}}$. 
	Since $x^\ast_\lambda$ was arbitrary, this shows that the row span contains $f_\lambda$. 
	Since $\lambda$ was arbitrary in $\Lambda$, this shows that the row span contains $V_\Lambda(\Xcal)$. 
	The reverse inclusion is direct. 
	The matrix $A$ has $1 + \sum_{\lambda\in\Lambda\setminus\emptyset} \prod_{i\in\lambda}(|\Xcal_i|-1)$ linearly independent rows, including a row of ones. This implies the dimension statement. 
\end{proof}

We now introduce the concept of $\Lambda$-balls, 
which we will use for constructing slicings of hierarchical models and formulating our theorems later on. 

\begin{definition}
	Let $\Xcal=\Xcal_1\times\cdots\times\Xcal_n$. 
	The $\Lambda$-ball in $\Xcal$ centered at $x\in\Xcal$ is the set of vectors that differ from $x$ exactly at the entries from some $\lambda\in\Lambda$, 
	\begin{equation*}
	K_\Xcal(x,\Lambda) : = \{ x'\in\Xcal\colon \{i\in[n] \colon x'_i\neq x_i \} =\lambda\in\Lambda \}. 
	\end{equation*}
	Note that all $\Lambda$-balls in $\Xcal$ have the same cardinality, regardless of their center, 
	\begin{equation*}
	K_\Xcal(\Lambda):= |K_\Xcal(x,\Lambda)| = 1 + \sum_{\lambda\in\Lambda\setminus \emptyset} \prod_{i\in\lambda}(|\Xcal_i|-1). 
	\end{equation*}
	We will drop the subscript $\Xcal$ when it is clear from the context. 
\end{definition}
An important special case of $\Lambda$-balls are Hamming balls. 
Recall that the Hamming distance between two vectors $x,x'\in\Xcal$ is defined as $d_H(x,x'):=|\{ i\in[n]\colon x_i\neq x'_i \}|$. 
The radius-$k$ Hamming ball in $\Xcal$ centered at $x$ is the set of vectors that differ from $x$ at most in $k$ entries, 
$K(x,\Lambda_k)=\{ x'\in\Xcal\colon d_H(x,x')\leq k \}$. 

In later sections we will consider slicings of a matrix $A$ with row span $V_\Lambda(\Xcal)$ into blocks corresponding to $\Lambda$-balls. 
We will use the following lemma: 
\begin{lemma}
	\label{lemma:fullranksubmatrix}
	Let $\Xcal=\Xcal_1\times\cdots\times\Xcal_n$, let $A$ be a matrix with row span $V_\Lambda(\Xcal)$, let $x\in\Xcal$, and let $K(x,\Lambda)$ be the $\Lambda$-ball in $\Xcal$ centered at $x$. 
	Then $A_{K(x,\Lambda)}$ has full rank, 
	\begin{equation*}
	\rank (A_{K(x,\Lambda)}) = \dim (V_\Lambda(\Xcal)) = 1+ \sum_{\lambda\in\Lambda\setminus\emptyset} \prod_{i\in\lambda}(|\Xcal_i|-1)
	=K(\Lambda).
	\end{equation*} 
\end{lemma}

\begin{proof}
	Consider the matrix entries $A_{(\lambda, \tilde x_\lambda), x}$ defined in Equation~\ref{eq:suffstatinteraction}. 
	For any $\lambda\not\in\Lambda$ and $\tilde x_\lambda\in \tilde{\Xcal}_\lambda = \times_{i\in\lambda} \Xcal_i\setminus\{0\}$, 
	we have that $A_{(\lambda, \tilde x_\lambda), x} =0$ for all $x\in K(0,\Lambda)$. 
	On the other hand, by Proposition~\ref{proposition:base} the full matrix $(A_{(\lambda,\tilde x_\lambda), x})_{(\lambda\in2^{[n]}, \tilde x_\lambda\in\tilde{\Xcal}_\lambda), x\in\Xcal}$ has rank equal to $\dim(V_{2^{[n]}}(\Xcal)) =|\Xcal|$ and thus it has row span $\R^\Xcal$. 
	This implies that the matrix 
	$(A_{(\lambda,\tilde x_\lambda), x})_{(\lambda\in\Lambda, \tilde x_\lambda\in\tilde{\Xcal}_\lambda), x\in K(0,\Lambda)}$ has row span $\R^{K(0,\Lambda)}$. 
	This proves the claim for the case of a $\Lambda$-ball centered at the zero vector. 
	The other cases follow from this after relabeling the states. 
\end{proof}

Lemma~\ref{lemma:fullranksubmatrix} states that certain collections of columns of the sufficient statistics matrix of a hierarchical model are linearly independent. 
This property is related to the notion of Kruskal rank, defined as the largest $r$ for which any $r$ columns of a matrix are linearly independent, which has been used before to study the rank of Khatri-Rao products~\cite{doi:10.1137/060661685}.

Kronecker products of hierarchical models correspond to hierarchical models with interaction sets given by the Cartesian product of the interaction sets of the factors. 
More precisely, if $A$ has row span $V_\Lambda(\Xcal_1\times\cdots\times\Xcal_n)$ and $B$ has row span $V_{\Lambda'}(\Ycal_1\times\cdots\times\Ycal_m)$, 
then $A\otimes B$ has row span $V_{\Lambda\times\Lambda'}(\Xcal\times\Ycal)$, 
where $\Lambda\times\Lambda'=\{ \lambda\times\lambda' \subseteq[n]\times[m] \colon \lambda\in \Lambda, \lambda'\in\Lambda' \}$. 
Figure~\ref{fig0} shows various types of examples. 
In the top row of the figure, the visible factor $\Ecal_A$ is an independence model. 
In the bottom row $\Ecal_A$ is an interaction model. 
The first column shows examples where the hidden factor $\Ecal_B$ is the set of all strictly positive distributions on $\Ycal$. 
These correspond to mixture models of $\Ecal_A$. We will cover them in Section~\ref{section:secants}. 
The second column shows examples where $\Ecal_B$ is an independence model. 
These correspond to Hadamard products of mixture models of $\Ecal_A$. We will cover them in Section~\ref{section:Hadamard}. 
The third column shows examples where $\Ecal_B$ is an interaction model. We will cover them in Section~\ref{section:harmonium}. 

\begin{figure}[t]
	\centering
	\scalebox{.85}{
		\begin{tabular}{ccc}
			\includegraphics[]{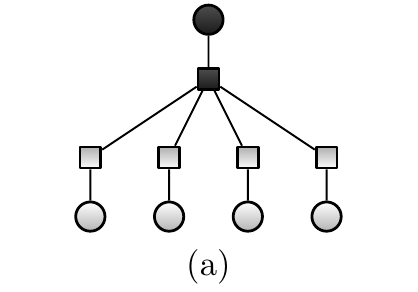} & \includegraphics[]{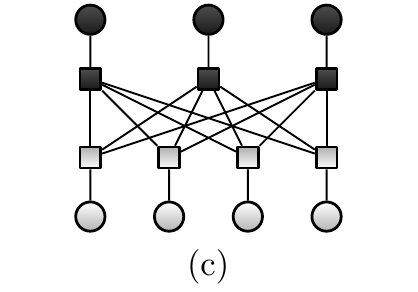} & \includegraphics[]{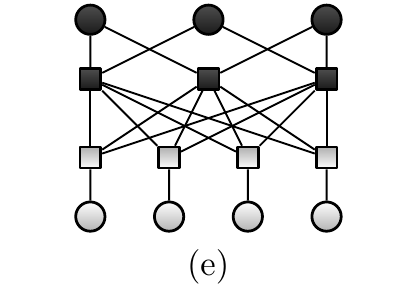}\\
			\includegraphics[]{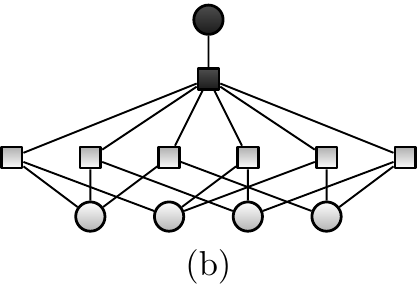} & \includegraphics[]{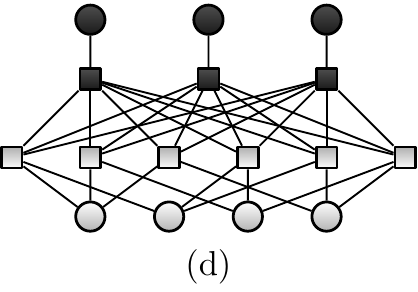} & \includegraphics[]{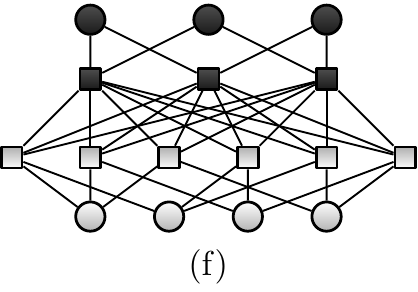}
		\end{tabular}
	}
	\caption{
		Examples of Kronecker product hierarchical models. 
		The dark and light circles represent hidden and visible variables, respectively. 
		The dark and light squares represent interactions among the adjacent hidden and visible variables, respectively. 
		The edges between squares represent interactions between the hidden and visible variables adjacent to those squares. 
		The graph between squares is full bipartite. 
		(a)~mixture model of an independence model, 
		(b)~mixture model of a pairwise interaction model, 
		(c)~Hadamard product of three mixture models of an independence model (restricted Boltzmann machine with three hidden units), 
		(d)~Hadamard product of three mixture models of a pairwise interaction model,  
		(e)~pairwise interaction mixture model of an independence model, 
		(f)~pairwise interaction mixture model of a pairwise interaction model. 
		From top to bottom and from left to right the models are more general. 
		From left to right these examples are covered in Theorems~\ref{theorem:mixtureskinteraction},~\ref{theorem:indhidintvis}~and~\ref{theorem:generalKronecker}. 
	}
	\label{fig0}
\end{figure}

\section{Mixture Models}
\label{section:secants}

\begin{definition}
	\label{eq:defmixture}
	The $k$-mixture of $\Ecal_A$ consists of all possible convex combinations of $k$ probability distributions from~$\Ecal_A$; that is, the probability distributions 
	\begin{gather*}
	p(x) = \sum_{i\in[k]}  \alpha(i) p^{(i)}(x), \quad x\in\Xcal, 
	\intertext{where}
	p^{(1)},\ldots,p^{(k)}\in\Ecal_A, \quad \alpha(1),\ldots, \alpha(k)\geq 0, \quad \sum_{i\in[k]}\alpha(i)=1. 
	\end{gather*}
\end{definition}
In algebraic geometry one usually considers secants instead of mixtures, in which case the weights $\alpha(i)$ in Definition~\ref{eq:defmixture} add to one but are not required to be non-negative. 
The resulting set is the union of possible affine hulls of $k$ points in $\Ecal_A$. 
The Zariski closure of the $k$-mixture of $\Ecal_A$ is $k$-th secant variety of~$\Ecal_A$. 
A standard reference on secant varieties is~\cite{zaktangents}. 

Mixtures of exponential families can be expressed as Kronecker product models:  
\begin{proposition}
	\label{proposition:base}
	Let $B$ have row span $\R^\Ycal$ and let $A$ include a row of ones. 
	Then the Kronecker product model $\Mcal_{A\otimes B}$ is the $|\Ycal|$-mixture of~$\Ecal_A$. 
\end{proposition}

\begin{proof}
	Without loss of generality let $B$ be the $|\Ycal|\times|\Ycal|$ identity matrix $I$. 
	The distributions of the exponential family $\Ecal_{A\otimes B}$ have the form $p_\theta(x,y) 
	= \frac{1}{Z(\theta)} \exp(\langle \Theta^\top B_y, A_x \rangle)  
	= \frac{1}{Z(\theta)} \exp(\langle \Theta^\top_y, A_x \rangle)$. 
	We may assume that the first row of $A$ is a vector of ones. 
	Hence, adding a suitable $\tilde\theta$ to the parameter vector, we obtain 
	$p_{\theta+\tilde{\theta}} (x,y) = \frac{1}{Z(\theta+\tilde\theta)}$ $\exp(\tilde \Theta^\top_{1,y}) \exp(\langle \Theta^\top_y, A_x \rangle)$. 
	The first term can be adjusted to obtain the mixture weights $\alpha$ from Equation~\ref{eq:defmixture} and the second term can be chosen independently for each value of $y$. 
\end{proof}

If $B$ has row span $\R^\Ycal$, we may assume that $B$ is the $|\Ycal|\times|\Ycal|$ identity matrix $I$. 
By Equation~\ref{eq:tropicalKronecker}, 
the tropical morphism has the following form: 
\begin{equation*}
\Acal_\theta = I\odot A_{h^{-1}_\theta}  =
\left(\begin{array}{c}
\sm{\boxed{A_{C_1}}& & &\\ & \boxed{A_{C_2}} & &\\ &&\ddots&\\ & & & \boxed{A_{C_{|\Ycal|}}}}
\end{array}
\right). 
\end{equation*}
In particular, the rank is just the sum of the ranks of the individual diagonal blocks, 
\begin{equation*}
\rank(\Acal_\theta) = \sum_{y\in\Ycal} \rank(A_{C_y}). 
\end{equation*}
In order to estimate the maximum of $\rank(\Acal_\theta)$ over $\theta$, an obvious strategy is to search for a slicing that produces as many full-rank blocks $A_{C_y}$ as possible. 
When $B$ is an identity matrix, the slicings can be constructed using any matrix $\Theta^\top B =\Theta^\top\in\R^{a\times b}$. 
By Lemma~\ref{lemma:fullranksubmatrix}, if $A$ is the sufficient statistics matrix of $\Ecal_\Lambda$ and $C_y$ contains, or is contained in, a $\Lambda$-ball, 
then the block $A_{C_y}$ has full rank. 
As we will show below, inference regions $C_y$ containing interaction balls can be obtained as slicings of $A$ by matrices of the form $\Theta^\top B = (A_{c_y})_y$, 
where $c_y$ are the centers of disjoint $\Lambda$-balls. 
We focus on $k$-interaction models. 

\begin{lemma}
	\label{lemma:asd}
	Let $\Xcal=\Xcal_1\times\cdots\times\Xcal_n$, 
	let $A$ have row span $V_k(\Xcal)$, 
	and let $B$ have row span $\R^\Ycal$. 
	Given $|\Ycal|$ disjoint $\Lambda_k$-balls in $\Xcal$, denoted $K(c_y,\Lambda_k)$, $y\in\Ycal$, 
	there is a $B$-slicing of $A$ with $C_y\supseteq K(c_y,\Lambda_k)$ for all $y\in\Ycal$. 
\end{lemma}

\begin{proof}
	Without loss of generality we choose a matrix $A$ with entries 
	$A_{(\lambda, \tilde x_\lambda), x}$ equal to $1$ if $x_\lambda =\tilde x_\lambda$ and 
	$-1$ otherwise, for $\lambda\in\Lambda_k$, $\tilde x_\lambda\in\Xcal_\lambda$, and $x\in\Xcal$. 
	Denoting the number of rows by $a=\sum_{\lambda\in\Lambda_k}\prod_{i\in\lambda}|\Xcal_i|$ we have that $\langle A_x,A_{x'}\rangle = a - 2d_H(A_x,A_{x'})$. 
	Furthermore, $d_H(A_x,A_{x'}) = 2|\{ \lambda\in\Lambda_k\colon x_\lambda \neq x'_\lambda \}|$. 
	If $x'\in K(x,\Lambda_k)$, then $d_H(A_x, A_{x'})\leq 2( 2^{k}-1)$. 
	On the other hand, if $x'\not\in K(x,\Lambda_k)$, then 
	$d_H(A_x, A_{x'})\geq 2( 2^{k}-1) + 1$. 
	Hence choosing $\Theta$ such that $\Theta^\top B_y = A_{c_y}$ for all $y\in\Ycal$, yields $C_y\supseteq K(c_y,\Lambda_k)$ for all $y\in\Ycal$. 
\end{proof}

\begin{theorem}
	\label{theorem:mixtureskinteraction}
	Let $\Xcal=\Xcal_1\times\cdots\times\Xcal_n$, 
	let $A$ have row span $V_k(\Xcal)$, 
	and let $B$ have row span $\R^\Ycal$. 
	\begin{itemize}
		\item 
If $\Xcal$ contains $|\Ycal|$ disjoint radius-$k$ Hamming balls, then 
		\begin{equation*}
		\dim(\TMcal_{A\otimes B}) = |\Ycal| \left(1 + \sum_{\lambda\in\Lambda_k\setminus\emptyset} \prod_{i\in\lambda} (|\Xcal_i|-1) \right)   -1.
		\end{equation*} 
\item 
		If $\Xcal$ can be covered by $|\Ycal|$ radius-$k$ Hamming balls, then 
		$$\dim(\TMcal_{A\otimes B}) = |\Xcal| -1 . $$
	\end{itemize}
\end{theorem}

\begin{proof}
	If $K(c_y, \Lambda_k)$, $y\in\Ycal$, are disjoint, 
	then by Lemma~\ref{lemma:asd} we obtain a slicing with $C_y \supseteq K(c_y, \Lambda_k)$  for all $y\in\Ycal$. 
	By Lemma~\ref{lemma:fullranksubmatrix} we have $\rank(A_{C_y}) = \rank(A)$ for all $y\in\Ycal$. 
	This yields $\rank(\Acal_\theta)=\sum_y\rank(A_{C_y}) = |\Ycal|\rank(A)$. 
	If $K(c_y, \Lambda_k)$, $y\in\Ycal$, cover $\Xcal$, 
	then $C_y \subseteq K(c_y,\Lambda_k)$ and $\rank(A_{C_y}) = |C_y|$ for all $y\in\Ycal$. 
	This yields $\rank(\Acal_\theta)=\sum_y\rank(A_{C_y}) = \sum_y|C_y|=|\Xcal|$. 
\end{proof}

We note the following special case where $\Ecal_A$ is an independence model. 
This case has been covered previously in Draisma's tropical approach to secant dimensions~\cite{Draisma}. 
The corresponding implications for the dimension of (not tropical) mixtures of independence models have also been studied before in algebraic geometry and tensor analysis; see~\cite{Catalisano2002263,AboOttavianiPeterson,landsbergtensors}. 

\begin{corollary}
	\label{theorem:secants-independence-model}
	Let $\Xcal=\Xcal_1\times\cdots\times\Xcal_n$, 
	let $A$ have row span $V_1(\Xcal)$, and let $B$ have row span $\R^\Ycal$. 
	\begin{itemize}
		\item
		If $\Xcal$ contains $|\Ycal|$ disjoint radius-one Hamming balls, then 
		\begin{equation*}
		\dim ( \TMcal_{A\otimes B} ) = |\Ycal| \left( 1 + \sum_{i\in[n]}(|\Xcal_i|-1) \right) - 1 .
		\end{equation*}
		\item 
		If $\Xcal$ can be covered by $|\Ycal|$ radius-one Hamming balls, then 
		\begin{equation*}
		\dim ( \TMcal_{A\otimes B} ) = |\Xcal|-1 .
		\end{equation*}  
	\end{itemize}
\end{corollary}

\section{Hadamard Products}
\label{section:Hadamard}

\begin{definition}
	The Hadamard product $\Mcal_1\ast\Mcal_2$ of two probability models $\Mcal_1, \Mcal_2$ on $\Xcal$ is the set of all probability distributions of the form 
	\begin{gather*}
	(p\ast q)(x) = \frac{p(x)q(x)}{\sum_{x'\in\Xcal} p(x')q(x')}, \quad x\in\Xcal, \quad \text{where }
	p\in\Mcal_1\text{ and } q\in\Mcal_2. 
	\end{gather*}
	The Hadamard product of $\Mcal_1,\ldots, \Mcal_m$ is defined in an analogous way. 
\end{definition}

Consider $m$ independent hidden variables, each with an associated exponential family. 
In other words, let $B^j\in\R^{b_j\times \Ycal_j}$, $j\in[m]$, and let $B\in\R^{b\times \Ycal}$ be the matrix with columns $B_y = (B^{1}_{y_1};\ldots;B^{m}_{y_m})$, $y\in\Ycal$.  
The corresponding exponential family factorizes as $\Ecal_B=\Ecal_{B^1}\otimes\cdots\otimes\Ecal_{B^m}$, where $\Ecal_{B^j}$ is an exponential family on $\Ycal_j$, for each $j\in[m]$. 
A Kronecker product model with independent hidden variables is the Hadamard product of the marginal models for the individual hidden variables: 

\begin{proposition}
	Let $\Ycal=\Ycal_1\times\cdots\times\Ycal_m$ and let $B=(B^{1};\ldots; B^{m})$ with $B^{j}_y=B^{j}_{y_j}$, $y\in\Ycal$. 
	Then $\Mcal_{A\otimes B} = \Mcal_{A\otimes B^1}\ast \cdots \ast \Mcal_{A\otimes B^m}$. 
\end{proposition}

\begin{proof}
	Consider a parameter vector $\theta = (\theta^1;\ldots;\theta^m)\in\R^{a\cdot b}$ with blocks corresponding to the blocks of $B=(B^1;\cdots;B^m)$. 
	We have 
	\begin{align*}
	p(x) 
	=& \sum_{y\in\Ycal} \frac{1}{Z(\theta)} \exp(  \langle \theta , A_x\otimes B_y\rangle )\\
	=& \sum_{y_1\in\Ycal_1}\cdots\sum_{y_m\in\Ycal_m} \frac{1}{Z(\theta)} \exp( \sum_{j\in[m]} \langle \theta^j, A_x\otimes B^j_{y_j}\rangle )\\
	=& \frac{1}{Z(\theta)} \prod_{j\in[m]}\sum_{y_j\in\Ycal_j} \exp( \langle \theta^j, A_x\otimes B^j_{y_j}\rangle ). 
	\end{align*}
	This proves the claim. 
\end{proof}

The tropical morphism of a Hadamard product decomposes into individual factor parts: 
\begin{lemma}\label{tropicalRBM}
	Let $\Ycal=\Ycal_1\times\cdots\times\Ycal_m$, 
	and let 
	$B=(B^{1};\ldots; B^{m})$ with $B^{j}_y=B^{j}_{y_j}$, $y\in\Ycal$. 
	Then 
	$\Acal_\theta = (\Acal_{\theta^1};\cdots;\Acal_{\theta^m})$, where $\Acal_{\theta^j} = A\odot B^j_{{h}_{\theta^j}}$, $\theta =(\theta^1,\ldots, \theta^m) \in\R^{a\cdot b}$. 
\end{lemma}

\begin{proof}
	We divide the parameter vector as $\theta = (\theta^1,\ldots, \theta^m)$ according to the blocks of $B=(B^1;\cdots; B^m)$, 
	such that $\langle\theta,  A\otimes B \rangle = \sum_{j\in[m]}\langle \theta^j , A\otimes B^j \rangle$. 
	For any given visible state $x\in\Xcal$ we have 
	\begin{align*}
	\max\left\{\langle \theta, A_x\otimes B_y \rangle \colon y\in\Ycal\right\} 
	=& \sum_{j\in[m]} \max\{\langle \theta^j, A_x\otimes B^j_{y_j}\rangle\colon y_j\in\Ycal_j\} \\
	=& \sum_{j\in[m]} \langle\theta^j, A_x \otimes B^j_{h_{\theta^j}(x)} \rangle \\
	=&  \sum_{j\in[m]} \langle{\theta^j}, \Acal_{\theta^j}\rangle,
	\end{align*}
	where $h_{\theta^j}$ is the inference function with parameter $\theta^j$ that maps each visible state $x$ to the most likely state $y_j$ of the $j$-th hidden variable. 
	This completes the proof. 
\end{proof}

In the following we focus on the case where each $B^j$ has row space $\R^{\Ycal_j}$. 
In this case each Hadamard factor of the Kronecker product model $\Mcal_{A\otimes B}$ is a mixture model of $\Ecal_A$. 
Choosing each $B^j$ equal to the $|\Ycal_j|\times |\Ycal_j|$ identity matrix, the tropical morphism takes the form 
\begin{equation}
\Acal_\theta =
\left(\begin{array}{c}
\sm{\boxed{\quad A_{C^1_1}\quad} & \\ & \boxed{\quad A_{C^1_2}\quad} } 
\\\midrule[0.03em]
\vdots \\\midrule[0.03em]
\sm{\boxed{A_{C^m_1}}& & \\ & \boxed{A_{C^m_2}} & \\ & & \boxed{A_{C^m_3}}}
\end{array}
\right). 
\label{proposition:rank-Hadamard-disjoint}
\end{equation}
Here $C^j_{y_j}:=h^{-1}_{\theta^j}(y_j)$, $y_j\in\Ycal_j$, is the slicing of $A$ by $B^j$, for all $j\in[m]$. 
An obvious strategy to maximize the rank is to construct the $m$ slicings of $A$ in such a way that we obtain as many disjoint sets $C^j_{y_j}$ with full rank $A_{C^j_{y_j}}$ as possible.

We present a construction based on truncated slicings, where each slicing divides the columns of $A$ in two sets, and then subdivides one of the two sets. 

\begin{lemma}
	\label{lemma:truncatedslicing}
	Let $A$ be some matrix. Let $I_M$ denote the $M\times M$ identity matrix. 
	Let $C_1,\ldots, C_{N}$ be an $I_N$-slicing of $A$ and let $D_1, D_2$ be an $I_2$ slicing of $A$. 
	Then $D_2\cap C_1,\ldots, D_2\cap C_N, D_1$ is a $I_{N+1}$ slicing of~$A$. 
\end{lemma}

\begin{proof}
	Let $B' =(B'_1,\ldots, B'_N)$ produce the slicing $C_1,\ldots, C_N$. 
	This means that $\langle A_x, B'_i -B'_j \rangle > 0$ for all $j\neq i$ iff $x\in C_i$, for all $i\in [N]$. 
	Let $B''=(B''_1, B''_2)$ produce the slicing $D_1,D_2$. 
	This means that $\langle A_x, B''_1-B''_2\rangle >0$ iff $x\in D_1$. 
	Note that $cB''$ produces the same slicing, for any $c>0$. 
	
	Now consider the matrix $B''' = (B'_1+c B''_2, \ldots, B'_N + c B''_2, c B''_1)$. 
	Let $E_1,\ldots, E_{N+1}$ denote the $B'''$-slicing of $A$. 
	We have 
	$\langle A_x, B'''_{N+1} - B'''_i\rangle = \langle A_x, c B''_1 - (B'_i + c B''_2) \rangle = c \langle A_x, B''_1-B''_2 \rangle - \langle A_x, B'_i \rangle$ for $i\leq N$. 
	Choosing $c>0$ large enough, this is positive iff $x\in D_1$. 
	This means that $E_{N+1}=D_1$. 
	
	For any $i\leq N$,  we have 
	$\langle A_x, B'''_{i} - B'''_j\rangle = \langle A_x, (B'_i + cB''_2) - (B'_j + cB''_2)\rangle = \langle A_x, B'_i - B'_j \rangle$ for all $j\leq N$, 
	and $\langle A_x, B'''_i - B'''_{N+1}\rangle =\langle A_x, (B'_i + c B''_2) - cB''_1\rangle =\langle A_x, B'_i\rangle + c \langle A_x, B''_2 - B''_1\rangle$. 
	Choosing $c>0$ large enough, all of these expressions are positive iff $x\in C_i \cap D_2$. 
	This means that $E_i=D_2\cap C_i$ for all $i\leq N$. 
\end{proof}

\begin{theorem}
	\label{theorem:indhidintvis}
	Let $\Xcal=\Xcal_1\times\cdots\times\Xcal_n$ 
	and let $A$ have row span $V_k(\Xcal)$ for some $1\leq k\leq n$. 
	Let $\Ycal=\Ycal_1\times\cdots\times\Ycal_m$ and let $B$ have row span $V_1(\Ycal)$. 
	\begin{itemize}
		\item 
		If $\Xcal$ contains $m$ disjoint Hamming balls $K^1,\ldots, K^m$, whereby $K^j$ contains $|\Ycal_j|-1$ disjoint radius-$k$ Hamming balls for $j=1,\ldots, m$, 
		and $\rank(A_{\Xcal\setminus(\cup_j K^j)}) = \rank(A)$, then 
		\begin{equation*}
		\dim(\TMcal_{A\otimes B}) = \Big(1+\sum_{\lambda\in \Lambda_k\setminus\emptyset} \prod_{i\in\lambda}(|\Xcal_i|-1) \Big) \Big(1+\sum_{j\in[m]}(|\Ycal_j|-1) \Big)-1.
		\end{equation*}
		\item 
		If $\Xcal$ can be covered by Hamming balls $K^1,\ldots, K^m$, such that $K^j$ can be covered by $|\Ycal_j|-1$ radius-$k$ Hamming balls, then 
		\begin{equation*}
		\dim(\TMcal_{A\otimes B}) = |\Xcal|-1.
		\end{equation*}
	\end{itemize}
\end{theorem}

\begin{proof}
	Have Equation~\ref{proposition:rank-Hadamard-disjoint} in mind. 
	Consider some $j\in[m]$. 
	We use Lemma~\ref{lemma:truncatedslicing} with $N=|\Ycal_j|-1$. 
	Let $D_2=K^j$ and let $C_1,\ldots, C_N$ be the slicing of $A$ by the matrix $(A_{c^j_{i}})_{i\in[N]}$, where $c^j_i$ are the centers of $|\Ycal_j|-1$ disjoint radius-$k$ Hamming balls in $K^j$. 
	This shows that there is a $B^j$-slicing of $A$ with blocks $C^j_1, \ldots, C^j_{|\Ycal_j|-1}, C^j_{|\Ycal_j|}$, where $C^j_{i}$ is contained in $K^j$ and contains a radius-$k$ Hamming ball, for all $1\leq i\leq |\Ycal_j|-1$, and $C^j_{|\Ycal_j|} = \Xcal\setminus K^j$. 
	
	Since all $K^j$ are disjoint, we have that all $C^j_i$, $1\leq i\leq |\Ycal_j|-1$, $j\in[m]$, are disjoint. 
	Since each of these sets contains a radius-$k$ Hamming ball, Lemma~\ref{lemma:fullranksubmatrix} implies $\rank(A_{C^j_{i}}) = \rank(A)$, for all $1\leq i\leq |\Ycal_j|-1$, $j\in[m]$. 
	The remaining columns of the tropical morphism have rank at least $\rank(A)$, since it is assumed that $\rank(A_{\Xcal\setminus(\cup_j K^j)})=\rank(A)$. 
	Hence we have $\rank(\Acal_\theta) = \sum_{j\in [m]} (|\Ycal_j|-1)\rank(A)  + \rank(A)$. This proves the first item. 
	The second item follows from similar arguments as the second item of Theorem~\ref{theorem:mixtureskinteraction}. 
\end{proof}

We note the following special case, where both $\Ecal_B$ and $\Ecal_A$ are independence models, which is known as a restricted Boltzmann machine. 

\begin{corollary}
	\label{corollary:TRBM}
	Let $\Xcal=\Xcal_1\times\cdots\times\Xcal_n$ 
	and let $A$ have row span $V_1(\Xcal)$. 
	Let $\Ycal=\Ycal_1\times\cdots\times\Ycal_m$ and let $B$ have row span $V_1(\Ycal)$. 
	\begin{itemize}
		\item 
		If $\Xcal$ contains $m$ disjoint Hamming balls $K^1,\ldots, K^m$, 
		whereby $K^j$ contains $|\Ycal_j|-1$ disjoint radius-one Hamming balls for $j=1,\ldots, m$, 
		and $\rank(A_{\Xcal\setminus(\cup_j K^j)}) = \rank(A)$, then 
		\begin{equation*}
		\dim(\TMcal_{A\otimes B}) = \Big(1+\sum_{i\in[n]}(|\Xcal_i|-1) \Big) \Big(1+\sum_{j\in[m]}(|\Ycal_j|-1) \Big)-1.
		\end{equation*}
		\item 
		If $\Xcal$ can be covered by Hamming balls $K^1,\ldots, K^m$, such that $K^j$ can be covered by $|\Ycal_j|-1$ radius-$k$ Hamming balls, then 
		\begin{equation*}
		\dim(\TMcal_{A\otimes B}) = |\Xcal|-1.
		\end{equation*}
	\end{itemize}
\end{corollary}

A weaker version of Corollary~\ref{corollary:TRBM} was obtained previously in~\cite{montufar2013discrete}. 
That result was based on slicings by parallel hyperplanes, which are less efficient than our construction with truncated slicings. 
One should note, however, that in order to realize slicings by parallel hyperplanes it is not required that each $B^j$ has row space $\R^{\Ycal_j}$, but only that $B^j$ can be projected into an arbitrary set of collinear points. 
The special case of Corollary~\ref{corollary:TRBM} where all variables are binary, $\Xcal=\{0,1\}^n$, $\Ycal=\{0,1\}^m$, was obtained previously in~\cite{Cueto2010}. 
That case is not improved by the present analysis, since for binary variables the truncated slicings are just slicings by hyperplanes.

\section{Interacting Hidden Variables}
\label{section:harmonium}

Consider a matrix $B$ of rank $b$ in reduced row echelon form. 
In this case the tropical morphism has the form 
\newcommand{\pA}{\phantom{\!\!\!A_{C_{p_1}}}}
\begin{equation}
\Acal_\theta =
\left(\begin{array}{c}
\sm{\boxed{A_{C_{p_1}}}& \cdots & \boxed{\pA} &									&\boxed{\pA}& \cdots & \boxed{\pA} &			  & 							 &\boxed{\pA} &\cdots  &\boxed{\pA}\\ 
	& 			& 					  &\boxed{A_{C_{p_2}}}  &\boxed{\pA}& \cdots & \boxed{\pA} &			  & 							 &\boxed{\pA} &\cdots  &\boxed{\pA}\\ 
	&			&					  &								    &					&			  & 					& \ddots  &  						      &\vdots 		  & 		    &\vdots\\ 
	& 			&					  & 			   					&					&			  & 			  		& 			   &\boxed{A_{C_{p_{b}}}}&\boxed{\pA} &\cdots &\boxed{\pA}}
\end{array}
\right).  
\label{equation:rref}
\end{equation}
From this we see that $\rank(\Acal_\theta) \geq \sum_{r=1}^{\rank(B)} \rank(A_{C_{p_r}})$. 
Rearranging the columns of $B$ suitably, 
any subset $P \subseteq\Ycal$ with $|P|= \rank(B_{P}) = \rank(B)$ can be obtained as the set of pivots $p_1,\ldots, p_{\rank(B)}$ of the reduced row echelon from. 
For instance, Lemma~\ref{lemma:fullranksubmatrix} shows that, if $\Ycal=\Ycal_1\times\cdots\times\Ycal_m$ and $B$ has row span $V_{\Lambda'}(\Ycal)$, 
then $|P| = \rank(B_{P}) =\rank(B)$ whenever $P$ is a $\Lambda'$-ball in $\Ycal$. 
However, it may be difficult or even impossible to find a $B$-slicing of $A$ such that $\rank(A_{C_{p_r}}) =\rank(A)$, for all $1\leq r\leq \rank(B)$. 
Nevertheless, 
in order to show that the tropical model $\TMcal_{A\otimes B}$ has dimension equal to $\rank(A)\cdot \rank(B) -1$, 
it suffices to show that $\rank(A_{C_{p_r,1}}, \ldots, A_{C_{p_r,s_r}}) = \rank(A)$ for all $1\leq r\leq \rank(B)$, where $(p_r,1),\ldots, (p_r,s_r)$ index the columns of (the reduced row echelon form of) $B$ with a non-zero entry at the $r$-th position and zeros in all the next entries. 

The following theorem addresses the tropical dimension of a Kronecker product model with an arbitrary hierarchical model $\Ecal_B$ and a $k$-interaction model $\Ecal_A$. This result includes Theorems~\ref{theorem:mixtureskinteraction} and~\ref{theorem:indhidintvis} as special cases. In fact it relaxes the hypothesis of Theorem~\ref{theorem:indhidintvis}.

\begin{theorem}
	\label{theorem:generalKronecker}
	Let $\Xcal=\Xcal_1\times \cdots\times\Xcal_n$ 
	and let $A$ have row span $V_k(\Xcal)$.  
	Let $\Ycal=\Ycal_1\times\cdots\times \Ycal_m$ and let $B$ have row span $V_{\Lambda'}(\Ycal)$. 
	\begin{itemize}
		\item If $\Xcal$ contains $\rank(B)$ disjoint radius-$k$ Hamming balls, then 
		\begin{equation*}
		\dim(\TMcal_{A\otimes B}) = \rank(B)\cdot\rank(A) -1. 
		\end{equation*}
		\item If $\Xcal$ can be covered by $\rank(B)$ disjoint radius-$k$ Hamming balls, then 
		\begin{equation*}
		\dim(\TMcal_{A\otimes B}) = |\Xcal| -1. 
		\end{equation*}
	\end{itemize} 
\end{theorem}

\begin{proof}
	The first item is as follows. 
	Let $\Ycal=\Ycal_1\times\cdots\times\Ycal_m$ and let $B$ be the matrix defined in Proposition~\ref{lemma:row-space} with row span $V_\Lambda(\Ycal)$. 
	We can group the columns according to their largest non-zero entry. 
	This yields one block of columns for each possible $(\lambda, \tilde y_\lambda)$. 
	The columns in the block $(\lambda, \tilde y_\lambda)$ have a $1$ in the $(\lambda, \tilde y_\lambda)$-th entry and zeros all the next entries. 
	
	Let $\Xcal=\Xcal_1\times\cdots\times\Xcal_n$ and let $A$ be the matrix with row span $V_\Lambda(\Xcal)$, with entries $A_{(\lambda, \tilde x_\lambda), x}$ equal to $1$ if $x_\lambda = \tilde x_\lambda$ and $-1$ otherwise, for $\lambda\in\Lambda$, $\tilde x_\lambda\in\Xcal_\lambda$, and $x\in\Xcal$. 
	We consider the sufficient statistics matrix given by $(\mathds{1};A)\in \R^{a \times \Xcal}$, where $a=1 + \sum_{\lambda\in\Lambda\setminus\emptyset}|\Xcal_\lambda|$. 
	
	Let $c_1,\ldots, c_b\in \Xcal$ be the centers of $b$ disjoint $\Lambda$-balls in $\Xcal$. 
	Consider a parameter matrix $\Theta^\top\in\R^{a\times b}$ with columns $\Theta^\top_j = \kappa_j ( 2^{k+1} -a ; A_{c_j})$, $j=1,\ldots, b$, where $\kappa_j \gg \sum_{j' < j} \|\Theta^\top_{j'}\|_1$. 
	Then $\Theta^\top B_y = \kappa_j ( (2^{k+1} -a ; A_{c_j}) + V_y )$ for all $y$ in the $j$-th block of columns of $B$, where $V_y\in\R^a$ is some vector with $\|V_y\|_1\ll 1$.  
	In this case, we have that $\langle (1; A_x), \Theta^\top B_y\rangle = \kappa_j ( 2^{k+1} - a + \langle A_x, A_{c_y} \rangle + \epsilon)$. 
	This is positive if $\langle A_x, A_{c_j}\rangle > a-2^{k+1}$ and negative if $\langle A_x, A_{c_j}\rangle < a -2^{k+1}$.  
	Now, note that if $d_H(x, x')\leq k$, then $\langle A_x, A_{x'}\rangle \geq (a - 1) - 2(2^{k}-1) = a - 2^{k+1} + 1$, and if $d_H(x,x')>k$, then $\langle A_x, A_{x'}\rangle \leq (a-1) - 2(2^k) = a - 2^{k+1} -1$.

	In turn, we have that $\cup_{y\colon l(y) =j } C_{y} \supseteq K(c_j, \Lambda_k)$, for all $1\leq j\leq b$, 
	where $l(y)$ denotes the largest non-zero entry of $B_y$.  
	Lemma~\ref{lemma:fullranksubmatrix} then yields the claim. 
	
	The second item is as follows. 
	In this case $\cup_{y\colon l(y)=j}C_y\subseteq K(c_j, \Lambda_k)$, for all $1\leq j\leq b$. 
	Since $C_y \cap C_{y'} =\emptyset$ for all $y\neq y'$, the matrix $(A_{C_y})_{y\colon l(y) = j}$ has linearly independent columns, for all $1\leq j\leq b$. 
\end{proof}

\section{Binary Restricted Boltzmann Machines are Not Defective}
\label{section:binRBM}

In~\cite{Cueto2010} it was shown that the restricted Boltzmann machine with $n$ visible and $m$ hidden binary units has the expected dimension $\min\{2^n -1, (n+1)(m+1)-1 \}$ whenever $\{0,1\}^n$ contains $m+1$ disjoint radius-one Hamming balls, $m+1\leq \Amd(n,3)$, or when $\{0,1\}^n$ can be covered by $m+1$ radius-one Hamming balls, $m+1\geq \Kmd(n,1)$. 
This also follows from Theorems~\ref{theorem:indhidintvis} and~\ref{theorem:generalKronecker} as the special case where both $\Ecal_A$ and $\Ecal_B$ are independence models with binary variables. 
This leaves open the cases where $\Amd(n,3)< m+1 <\Kmd(n,1)$. 
An additional problem is that in general the functions $\Amd$ and $\Kmd$ can only be bounded but not evaluated exactly. 
In~\cite{Cueto2010} it was conjectured that the restricted Boltzmann machine always has the expected dimension. 
In this section we resolve that conjecture affirmatively. 

Our strategy is to bound the dimension of the RBM from below by the dimension of a mixture model. 
By the results from~\cite{Catalisano2011}, mixtures of binary independence models are not defective whenever the number of visible variables is at least $5$. 
In general RBMs do not contain mixture models of independence models with the same number of parameters, 
as was shown in~\cite{montufar2012does}. 
However, it is possible to relate the dimension of the two models. 
Here we consider the dimension of the actual model, not of its tropical version. 

The following lemma lower bounds the dimension of a Kronecker product model with independent binary hidden variables by the dimension of a mixture model with the same number of parameters. 

\begin{lemma}
	\label{theorem:mixthad}
	Let $B$ have row span $V_1( \{0,1\}^m )$ and let $\tilde B$ have row span $\R^{m+1}$. 
	Then $\dim(\Mcal_{A\otimes B}) \geq \dim(\Mcal_{A\otimes \tilde B})$. 
\end{lemma}
\begin{proof}
	Let $B$ be given by $B_y = (1; y_1; \ldots; y_m)$, $y\in\{0,1\}^m$, and let $\tilde B = I_{m+1}$ be the $(m+1)\times (m+1)$ identity matrix. 
	For any $j\in[m]$ let $e_j\in\{0,1\}^m$ denote the vector with a single one at the $j$-th position and zeros elsewhere. 
	Any $y\in\{0,1\}^m$ can be written as $y = \sum_{j\in[m] \colon e_j \leq y } e_j$. 
	Here $ y'\leq y$ if and only if $y'_j\leq y_j$ for all $j\in[m]$. 
	As discussed in Equation~\ref{eq:rankJ}, the Jacobian of the natural parametrization of a marginal model $\Mcal_{F}$ has rank $\rank(J_{\Mcal_{F}}(\theta)) = \rank \left(\sum_y p_\theta(y|x) F(x,y) \right)_x -1$. 
	We have 
	\begin{align*}
	\sum_{y\in\{0,1\}^m} p_\theta(y|x) (A_x \otimes B_y)
	=& A_x\otimes \sum_{y\in\{0,1\}^m} p_\theta(y|x) B_y  \\
	=& A_x \otimes \Big(  1; \sum_{y\in\{0,1\}^m} p_\theta(y|x) y \Big)\\
	=& A_x \otimes \Big(  1; \sum_{y\in\{0,1\}^m} p_\theta(y|x)  \sum_{j\in[m]\colon e_j\leq y } e_j \Big)\\
	=& A_x \otimes \Big(  1; \sum_{j\in[m]}  \sum_{y\in\{0,1\}^m\colon y\geq e_j} p_\theta(y|x) e_j \Big)\\
	=& A_x \otimes \Big(  1; \sum_{j\in[m]}  p_\theta(y_j=1 |x) e_j \Big) . 
	\end{align*}
	The conditional distributions are given by  
	\begin{align*}
	p_\theta(y_j =1 | x) 
	= \frac{\exp( \Theta_j A_x ) }{1 + \exp( \Theta_j A_x) } , \quad j\in[m], x\in\Xcal. 
	\end{align*}
	
	On the other hand, for a hidden unit with sufficient statistics matrix $\tilde B = I_{m+1}$ we have 
	\begin{align*}
	\sum_{j\in\{0,1,\ldots, m\} } \tilde p_{\tilde\theta}(j|x) (A_x \otimes \tilde B_j) 
	=& A_x \otimes \Big( \tilde p_{\tilde \theta}(0|x) ; \sum_{j\in[m]}  \tilde p_{\tilde \theta}(j|x) e_j \Big) . 
	\end{align*}
	In this case the conditional distributions are given by 
	\begin{align*}
	\tilde p_{\tilde\theta} ( j | x) 
	= \frac{\exp( \tilde \Theta_j A_x) }{ \sum_{j} \exp( \tilde \Theta_j A_x ) } ,\quad j\in\{0,1,\ldots, m\} , x\in\Xcal. 
	\end{align*}
	
	Without loss of generality let the first row of $A$ be $\mathds{1}$. 
	Consider the parameters $\Theta_{j,i } = \tilde \Theta_{j,i} - \tilde \Theta_{0,i}$, $i=2,\ldots, a$ and $\Theta_{j,1}  = \tilde \Theta_{j,1} - \tilde \Theta_{0,1}  - \gamma$. 
	For any $\epsilon>0$ we can choose $\gamma$ large enough such that 
	\begin{equation*}
	\sum_{j=1}^m \left| \exp(\gamma)  p_\theta(y_j =1 | x) 
	-  \frac{\tilde p_{\tilde \theta}(j | x)}{ \tilde p_{\tilde \theta}(0| x)} \right| \leq \epsilon, \quad\text{for all } x\in\Xcal. 
	\end{equation*}
	This implies that 
	\begin{align*}
	\dim(\Mcal_{A\otimes B}) + 1
	= &
	\max_\theta \rank\left(A_x \otimes \Big(  1; \sum_j  p_\theta(y_j=1 |x) {e_j} \Big)  \right)_x\\
	\geq&
	\max_{\tilde \theta} \rank\left(A_x \otimes \Big(1 ; \exp(-\gamma) \sum_{j=1}^{m}  \frac{\tilde p_{\tilde \theta}(j |x)}{\tilde p_{\tilde \theta}(0|x)} {e_j} \Big) 
	\right)_x \\
	=&
	\max_{\tilde \theta} \rank\left(A_x \otimes \Big(\tilde p_{\tilde \theta}(0|x) ; \sum_{j=1}^{m} \tilde p_{\tilde \theta}(j |x) {e_j} \Big) 
	\right)_x \\
	= & \dim(\Mcal_{A\otimes I_{m+1}}) +1. 
	\end{align*}
	This completes the proof. 
	See Example~\ref{example:lemma} in the Appendix for a more explicit formulation of this proof in the case $m=2$. 
\end{proof}

The dimension bound from Lemma~\ref{theorem:mixthad} is not always tight. 
For example, the $3$-mixture of $4$ independent binary variables is defective and has dimension $13$, 
whereas the RBM with $4$ visible and $2$ hidden binary units has the expected dimension $14$.

\begin{corollary}
	\label{corollary:dimrbm}
	Let $n$ and $m$ be non-negative integers. 
	The restricted Boltzmann machine with $n$ visible and $m$ hidden binary units has dimension $\min\{2^n-1, (n+1)(m+1)-1\}$. 
\end{corollary}

\begin{proof}
	Let $A$, $B$, $\tilde B$ be matrices with row span $V_1(\{0,1\}^n)$, $V_1(\{0,1\}^m)$,  $\R^{m+1}$, respectively. 
	Then $\Mcal_{A\otimes B}$ is the restricted Boltzmann machine with $n$ visible and $m$ hidden binary variables and 
	$\Mcal_{A\otimes \tilde B}$ is the $(m+1)$-mixture of the independence model of $n$ binary variables. 
	
	The work~\cite{Catalisano2011} shows that 
	$\dim(\Mcal_{A\otimes \tilde B}) = \min\{ 2^n -1,  (m+1) (n+1) -1 \}$ unless $(n,m)=(4,2)$. 
	Lemma~\ref{theorem:mixthad} then implies $\dim(\Mcal_{A\otimes B}) \geq \dim(\Mcal_{A\otimes \tilde B}) = \min\{ 2^n -1,  (m+1) (n+1) -1 \}$ whenever $(n,m)\neq(4,2)$. 
	Since this is also the maximum possible dimension, the bound is tight. 
	That the RBM with $(n,m)=(4,2)$ has the expected dimension is shown in~\cite{drton2009lectures,Cueto:2010:ICB:1866469.1866627}. 
\end{proof}

An RBM with $5$ visible units is the first case with a gap between the largest RBM previously known to have dimension equal to the number of parameters ($4$ hidden units and $29$ parameters) and the smallest RBM previously known to be full dimensional ($7$ hidden units and $47$ parameters suffice to define a $31$-dimensional subset of the $31$-simplex) as described by~\cite{Cueto2010}, which lists a collection of such gaps in Table 4.1 for $5\leq n\leq 512$. Corollary \ref{corollary:dimrbm} closes all such gaps.  Thus the smallest RBM with $5$ visible units which could possibly define a full-dimensional subset of the simplex, with $5$ hidden units and $35$ parameters, does so. These exampes can be tested by computing Jacobians with the Matlab code provided at~\texttt{http://personal-homepages.mis.mpg.de/montufar}.

Corollary~\ref{corollary:dimrbm} proves that a binary RBM always has the expected dimension. 
The conjecture posed in~\cite{Cueto2010} goes further and states that the tropical binary RBM always has the expected dimension. That question remains unsettled at this point.

\section{Conclusion}
\label{section:conclusion}

In this work we study the dimension of marginals of exponential families whose sufficient statistics matrix factorizes as the Kronecker product of a visible and a hidden sufficient statistics matrix, called Kronecker product models. 
The Jacobian of these models factorizes as a Khatri-Rao product of the visible sufficient statistics matrix and the expectation parameters of the hidden exponential family given the visible variables. 
The tropical morphism arises as the limit of the Jacobian when the natural parameters of the model are scaled by an infinitely large number. 
It is described by the inference functions of the hidden variables given the visible variables, which correspond to slicings of the visible sufficient statistics matrix by the normal fan of the hidden sufficient statistics matrix. 

Based on these geometric and combinatorial descriptions, we computed the tropical dimension of mixtures of interaction models and Hadamard products of mixtures of interaction models. 
These results extend previous work on secant dimensions, which are most often focused on Segre and Veronese varieties (corresponding to independence models and multinomial models). 
These results also generalize previous work on Hadamard products, which were focused on products of mixtures of independence models. 
Theorem~\ref{theorem:generalKronecker} generalizes this further to the case of Kronecker products of arbitrary hierarchical models and $k$-interaction models. 
Additionally, we showed that binary restricted Boltzmann machines always have the expected dimension, thus completing the dimension description of these models from~\cite{Cueto2010}. 

Our analysis leaves many questions unanswered. 
In this work we have focused on the case where the visible exponential family is a $k$-interaction model. 
The generalization to arbitrary visible hierarchical models is left for future work. 
Furthermore, similarly to~\cite{Draisma}, the tropical approach leads in many cases to combinatorial conditions that can be very difficult to verify outside of well established cardinality bounds for error correcting codes. 
We think that a promising direction is the formulation of dimension bounds in terms of simpler models, 
as done in Lemma~\ref{theorem:mixthad} for the case of a binary independence hidden model. 
Extending that result one could ask: when is the dimension of $\Mcal_{A\otimes B}$ bounded below by the dimension of $\Mcal_{A\otimes I}$, where $I$ is an identity matrix with the same rank as $B$? 

The factorization property of the Jacobian of Kronecker product models suggests to study the models of conditional probability distributions of the hidden variables given the visible variables in more detail. 
This is a manifold of tuples of exponential family distributions with natural parameters given by the linear projection of the sufficient statistics of the visible model. 
An analysis of the Kruskal ranks for these sets can be used to obtain bounds on the Kruskal rank of the Jacobian. 

Another interesting line of investigation is the classification of the support sets of distributions in the closure of Kronecker product models. 
For mixture models the problem is simple, when the support sets of the visible exponential family are known. 
For Hadamard products the problem has been studied in~\cite{montufar2012does} based on linear threshold codes, which are the images of inference functions. 
We think that studying the combinatorics of Kronecker product polytopes could yield helpful insights. 
Given two polytopes $P_A = \conv\{A_x\colon x\in\Xcal \}$ and $P_B=\conv\{ B_y\colon y\in\Ycal \}$, we define the Kronecker product polytope as $P_{A\otimes B} = \conv \{ A_x\otimes B_y\colon (x,y)\in\Xcal\times \Ycal \}$. 
Although this appears as a rather natural composition of polytopes, 
we are not aware of works studying such objects explicitly or in a principled way.

\appendix
\section*{Appendix}

\section{Examples}

\begin{example}\label{example:lemma}
	Here we give a more comprehensive version of the proof of Lemma~\ref{theorem:mixthad} for the special case where $m=2$. 
	Consider the matrices 
	\newcolumntype{T}{>{\scriptsize$}c<{$}}
	\newcommand{\bm}{\!\!}
	\newcommand{\enm}{\!\!}
	\begin{equation*}
	B = 
	\overset{
		\renewcommand{\arraystretch}{0.5}
		\renewcommand{\tabcolsep}{1.5pt}
		\begin{tabular}{T T T T }
		00 & 01 & 10 & 11
		\end{tabular}
	}{
	\renewcommand{\arraystretch}{1}
	\renewcommand{\arraycolsep}{3pt}
	\left(\begin{array}{c c c c}
	1 & 1 & 1 & 1\\ 
	0 & 0 & 1 & 1\\ 
	0 & 1 & 0 & 1\\ 
	\end{array}\right)
}
\quad
\text{and}
\quad
\tilde B = 
\overset{
	\renewcommand{\arraystretch}{0.5}
	\renewcommand{\tabcolsep}{5pt}
	\begin{tabular}{T T T}
	0 & 1 & 2
	\end{tabular}
}{
\renewcommand{\arraystretch}{1}
\left(\begin{array}{c c c}
1 & 0 & 0\\ 
0 & 1 & 0\\ 
0 & 0 & 1\\ 
\end{array}\right)
}.
\end{equation*}

We have 
\begin{align*}
\sum_{y\in\{0,1\}^2} p_\theta(y|x) (A_x \otimes B_y)
=& A_x\otimes \sum_{y\in\{0,1\}^m} p_\theta(y|x) B_y  \\
=& A_x \otimes 
\begin{pmatrix} 
1\\ 
p_\theta(10|x) + p_\theta(11|x) \\
p_\theta(01|x) + p_\theta(11|x) 
\end{pmatrix}\\
=& A_x \otimes 
\begin{pmatrix} 
1\\ 
p_\theta(y_1=1|x)\\
p_\theta(y_2=1|x) 
\end{pmatrix}. 
\end{align*}
By Equation~\ref{equation:conditionaldistributions} the conditional distributions are given by 
\begin{equation*}
p_\theta(y | x) = \frac{1}{Z(\Theta A_x)} \exp(\langle \Theta A_x, B_y \rangle)
\end{equation*}	
such that
\begin{align*}
p_\theta(y_1 =1| x) 
= & \frac{\exp( \langle\Theta A_x, \left(\begin{smallmatrix}
	1 \\ 1
	\end{smallmatrix}\right)\rangle)}{\exp( \langle\Theta A_x, \left(\begin{smallmatrix}
	1 \\ 0
	\end{smallmatrix}\right\rangle) + \exp( \langle\Theta A_x, \left(\begin{smallmatrix}
	1 \\ 1
	\end{smallmatrix}\right)\rangle)} \\
= & \frac{\exp(\Theta_1 A_x)}{1 + \exp(\Theta_1 A_x )} , 
\end{align*}
and similarly for $p_\theta(y_2=1|x)$. 

On the other hand, for a hidden unit with sufficient statistics matrix $\tilde B$ we have 
\begin{equation*}
\sum_{j\in\{0,1,2\}} \tilde p_{\tilde\theta}(A_x\otimes \tilde B_j) 
= A_x \otimes \begin{pmatrix}
\tilde p_{\tilde \theta}(0|x)\\
\tilde p_{\tilde \theta}(1|x)\\
\tilde p_{\tilde \theta}(2|x)
\end{pmatrix}. 
\end{equation*}
In this case the conditional distributions are given by 
\begin{equation*}
\tilde p_{\tilde \theta}(1 | x) = \frac{\exp(\tilde \Theta_1 A_x)}{\exp(\tilde \Theta_0 A_x) + \exp(\tilde \Theta_1 A_x) + \exp(\tilde \Theta_2 A_x) },
\end{equation*}
and similarly for $j=0$ and $j=2$. 

For any given 
\begin{equation*}
\renewcommand{\arraycolsep}{4pt}
\tilde \Theta
=
\left(\begin{array}{cccc}
\tilde \Theta_{0,1} & \tilde \Theta_{0,2} & \cdots & \tilde \Theta_{0,a}\\
\tilde \Theta_{1,1} & \tilde \Theta_{1,2} & \cdots & \tilde \Theta_{1,a}\\
\tilde \Theta_{2,1} & \tilde \Theta_{2,2} & \cdots & \tilde \Theta_{2,a}
\end{array}\right)
\end{equation*}
we can define 
\begin{equation*}
\renewcommand{\arraycolsep}{6pt}
\Theta
=
\left(\begin{array}{cccc}
\tilde \Theta_{0,1} -\tilde \Theta_{0,1} -\gamma & \tilde \Theta_{0,2} -\tilde \Theta_{0,2} &  \cdots & \tilde \Theta_{0,a} - \tilde \Theta_{0,a}\\
\tilde \Theta_{1,1} -\tilde \Theta_{0,1} -\gamma & \tilde \Theta_{1,2} -\tilde \Theta_{0,2} &  \cdots & \tilde \Theta_{1,a} - \tilde \Theta_{0,a}\\
\tilde \Theta_{2,1} -\tilde \Theta_{0,1} -\gamma & \tilde \Theta_{2,2} -\tilde \Theta_{0,2} &  \cdots & \tilde \Theta_{2,a} - \tilde \Theta_{0,a}\\
\end{array}\right). 
\end{equation*}
Without loss of generality assume that the first row of $A$ is a row of ones.  
We have then 
\begin{equation*}
\begin{pmatrix}
1\\
p_\theta(y_1=1|x)\\
p_\theta(y_2=1|x)\\
\end{pmatrix}
=  
\begin{pmatrix}
1\\
\frac{\exp(\Theta_1 A_x)}{1 + \exp(\Theta_1 A_x)}\\
\frac{\exp(\Theta_2 A_x)}{1 + \exp(\Theta_2 A_x)}
\end{pmatrix}
= 
\begin{pmatrix}
1\\
\frac{\exp( -\gamma + (\tilde \Theta_{1} - \tilde \Theta_{0}) A_x)}
{1 + \exp( -\gamma + (\tilde \Theta_{1} - \tilde \Theta_{0}) A_x)}\\
\frac{\exp( -\gamma + (\tilde \Theta_{2} - \tilde \Theta_{0}) A_x)}
{1 + \exp( -\gamma + (\tilde \Theta_{2} - \tilde \Theta_{0}) A_x)}\\
\end{pmatrix}. 
\end{equation*}
For $\gamma$ large enough we obtain an arbitrarily accurate approximation of the form  
\begin{equation*}
\begin{pmatrix}
1\\
p_\theta(y_1=1|x)\\
p_\theta(y_2=1|x)\\
\end{pmatrix}
\approx  
\begin{pmatrix}
1\\
\exp(-\gamma)\exp( (\tilde\Theta_1 -\tilde\Theta_0) A_x  )\\
\exp(-\gamma)\exp( (\tilde\Theta_2 -\tilde\Theta_0) A_x  )
\end{pmatrix} 
= 
\begin{pmatrix}
1\\
\exp(-\gamma)\frac{\tilde p_{\tilde \theta}(1|x)}{\tilde p_{\tilde \theta}(0|x)}\\
\exp(-\gamma)\frac{\tilde p_{\tilde \theta}(2|x)}{\tilde p_{\tilde \theta}(0|x)}
\end{pmatrix}. 
\end{equation*}

Hence for any $\tilde \theta$ there is a $\theta$ with 
\begin{align*}
\rank\left( 
A_x \otimes \begin{pmatrix}
\tilde p_{\tilde \theta}(0|x)\\
\tilde p_{\tilde \theta}(1|x)\\
\tilde p_{\tilde \theta}(2|x)\\
\end{pmatrix}\right)_x
= &
\rank
\left(
L 
\cdot 
\left(A_x \otimes  
\begin{pmatrix}
1 \\ 
\exp(-\gamma) \frac{\tilde p_{\tilde \theta}(1|x) }{\tilde p_{\tilde \theta}(0|x)}\\
\exp(-\gamma) \frac{\tilde p_{\tilde \theta}(2|x) }{\tilde p_{\tilde \theta}(0|x)}
\end{pmatrix} \right)_x \cdot R\right)\\
= & 
\rank\left(A_x \otimes \begin{pmatrix}
1 \\ 
\exp(-\gamma) \frac{\tilde p_{\tilde \theta}(1|x) }{\tilde p_{\tilde \theta}(0|x)}\\
\exp(-\gamma) \frac{\tilde p_{\tilde \theta}(2|x) }{\tilde p_{\tilde \theta}(0|x)}
\end{pmatrix} \right)_x  \\
\leq & 
\rank\left(\left(A_x \otimes \begin{pmatrix}
1 \\ 
\exp(-\gamma) \frac{\tilde p_{\tilde \theta}(1|x) }{\tilde p_{\tilde \theta}(0|x)}\\
\exp(-\gamma) \frac{\tilde p_{\tilde \theta}(2|x) }{\tilde p_{\tilde \theta}(0|x)}
\end{pmatrix} \right)_x  + \epsilon \right)\\
= & 
\rank\left(
A_x \otimes \begin{pmatrix}
1 \\ 
p_\theta(y_1 =1 |x)\\
p_\theta(y_2 =1 |x)
\end{pmatrix}
\right)_x , 
\end{align*}
where 
\begin{equation*}
L = 
\left(\begin{smallmatrix}
1& & \\
&\ddots& \\
& & 1
\end{smallmatrix}\right)\otimes 
\left(\begin{smallmatrix}
1 &  & \\
& \exp(\gamma) & \\
&  & \exp(\gamma)
\end{smallmatrix}\right)
\quad\text{and}\quad
R= \left(\begin{smallmatrix}
\tilde p_{\tilde \theta}(0|1) & & \\
& \ddots & \\
& & \tilde p_{\tilde \theta}(0|a)
\end{smallmatrix}\right). 
\end{equation*}

In turn 
\begin{equation*}
\max_{\tilde \theta}
\rank\left( 
A_x \otimes \begin{pmatrix}
\tilde p_{\tilde \theta}(0|x)\\
\tilde p_{\tilde \theta}(1|x)\\
\tilde p_{\tilde \theta}(2|x)\\
\end{pmatrix}\right)_x
\leq 
\max_\theta 
\rank\left(
A_x \otimes \begin{pmatrix}
1 \\ 
p_\theta(y_1 =1 |x)\\
p_\theta(y_2 =1 |x)
\end{pmatrix}
\right)_x 
\end{equation*}
and
\begin{equation*}
\dim(\Mcal_{A\otimes \tilde B}) \leq \dim(\Mcal_{A\otimes B}). 
\end{equation*}
\end{example}

\section{Simple Bounds for Error Correcting Codes}

Some of our results are formulated in terms of the maximal number of disjoint $\Lambda$-balls that can be fit in some $\Xcal=\Xcal_1\times\cdots\times\Xcal_n$. 
Let $\Amd(\Xcal, d)$ denote the maximal cardinality of a subset of $\Xcal$ of minimum Hamming distance at least $d$. 
If $\Xcal=\{0,1,\ldots, q-1\}^n$, we write $\Amd_q(n,d)$. 
Closed forms for these functions are known only in special cases. 
We recall the Gilbert-Varshamov bound~\cite{Gilbert:1952,Varshamov:1957}: 
$$\Amd_q(n,d)\geq \frac{q^n}{\sum_{j=0}^{d-1}{n\choose j} (q-1)^j}.$$ 
If $q$ is a prime power, then 
$$\Amd_q(n,d)\geq q^{n-1-\lfloor\log_q( \sum_{j=0}^{d-2}{n-1\choose j}(q-1)^j) \rfloor}.$$ 
A simple upper bound is the sphere packing bound, 
$\Amd_q(n, d) \leq q^n / K_q(t)$, $t=\left\lfloor\frac{d-1}{2}\right\rfloor$.

Theorem~\ref{theorem:indhidintvis} is formulated in terms of the maximal number of disjoint Hamming balls that fit in a larger Hamming ball. 
An estimate for this number can be given as follows. 
\begin{proposition}
	Let $\Xcal=\Xcal_1\times\cdots\times\Xcal_n$, and let $k\leq l\leq n$. 
	Then it is possible to fit $|K(l-k)|/|K(2k)|$ disjoint radius-$k$ Hamming balls in a radius-$l$ Hamming ball. 
\end{proposition}

\begin{proof}
	Denote by $C$ the set of centers of a largest possible collection of disjoint radius-$k$ Hamming balls contained in a radius-$l$ Hamming ball $K(0,l)$. 
	Consider the radius-$(l+k+1)$ sphere $S$. 
	Let $d=2k+1$. 
	For every $x\in K(0,l+k+1)$ there is a $c_x\in C\cup S$ such that $d_H(x,c_x)\leq d-1$. 
	This implies that $K(0, l+k+1)\subseteq\cup_{c\in C} K(c, d-1) \cup(K(0, l+k+1)\setminus K(0, l-k))$ and $K(0, l-k)\subseteq\cup_{c\in C}K(c, d-1)$. 
	Therefore, $|C|\geq K(l-k) / K(2k)$. 
\end{proof}

\paragraph{Acknowledgments}
	This work was supported in part by DARPA grant FA8650-11-1-7145. 
	Parts of this work were carried out while G.M.~was at the Pennsylvania State University. 

\bibliographystyle{abbrv} 
\bibliography{referenzen}

\end{document}